%% file: Main.tex
\documentclass[journal]{IEEEtran}
\usepackage{amsmath,amssymb,amsfonts}
\usepackage{algorithm}
\usepackage{algorithmic}
\usepackage{graphicx}
\usepackage{textcomp}
\usepackage{xcolor}
\usepackage[utf8]{inputenc} % allow utf-8 input
\usepackage[T1]{fontenc}    % use 8-bit T1 fonts
\usepackage{hyperref}      % hyperlinks
\usepackage{url}            % simple URL typesetting
\usepackage{booktabs}       % professional-quality tables
\usepackage{amsfonts}       % blackboard math symbols
\usepackage{nicefrac}       % compact symbols for 1/2, etc.
\usepackage{array,multirow,graphicx}
\usepackage{float}
\usepackage{comment}
\usepackage{epsfig}
\usepackage{amsmath}
\usepackage{autobreak}
\usepackage{amssymb}
\usepackage{indentfirst}
\usepackage{mathrsfs}
\usepackage{bm}        % for including graphics
\usepackage{psfrag}             % for changing text in eps graphics
\usepackage{subfigure}          % for subfigures
\usepackage{caption}
\usepackage{empheq}
\usepackage{enumerate}
\usepackage{soul}
\usepackage{bigints}
\DeclarePairedDelimiter{\ceil}{\lceil}{\rceil}
\DeclarePairedDelimiter\floor{\lfloor}{\rfloor}

\newenvironment{proof}[1][Proof]%
  {\smallskip\par\noindent\textbf{#1\,:\ }}%
  {\hspace*{\fill} \rule{6pt}{6pt}\smallskip}
\newenvironment{proof*}[1][Proof]%
  {\medskip\par\noindent\textbf{#1\,:\ }}%
\newtheorem{assumption}{\emph{\textbf{Assumption}}}
\newtheorem{condition}{\emph{\textbf{Condition}}}
\newtheorem{theorem}{\textbf{Theorem}}
\newtheorem{lemma}{\textbf{Lemma}}

\newcommand{\HB}[1]{{\color{red}#1}}
\newcommand{\KB}[1]{{\color{blue}#1}}

\ifCLASSINFOpdf
\else
\fi
\begin{document}
\title{Asynchronous Local Computations in \\ Distributed Bayesian Learning}

\author{Kinjal Bhar,~\IEEEmembership{}
        He Bai,~\IEEEmembership{} Jemin George~\IEEEmembership{} and~Carl~Busart~\IEEEmembership{}
\thanks{K. Bhar and H. Bai are with the Mechanical and Aerospace Engineering, Oklahoma State University, Stillwater,
OK 74072, USA.}
\thanks{J. George and C. Busart are with DEVCOM Army Research Lab, Adelphi, MD 20783, USA.}}


\markboth{}%
{Bhar \MakeLowercase{\textit{et al.}}}



\maketitle

\begin{abstract}
Due to the expanding scope of machine learning (ML) to the fields of sensor networking, cooperative robotics and many other multi-agent systems, distributed deployment of inference algorithms has received a lot of attention. These algorithms involve collaboratively learning unknown parameters from dispersed data collected by multiple agents. There are two competing aspects in such algorithms, namely, intra-agent computation and inter-agent communication. Traditionally, algorithms are designed to perform both synchronously. However, certain circumstances need frugal use of communication channels as they are either unreliable, time-consuming, or resource-expensive. In this paper, we propose  gossip-based asynchronous communication {to leverage fast computations and reduce communication overhead simultaneously.} %\HB{for what?} 
We analyze the effects of multiple (local) intra-agent computations by the active agents between successive inter-agent communications. For local computations, Bayesian sampling via unadjusted Langevin algorithm (ULA) MCMC is utilized. The communication is assumed to be over a connected graph (e.g., as in \emph{decentralized learning}), however, the results can be extended to coordinated communication where there is a central server (e.g., \emph{federated learning}). We theoretically quantify the convergence rates in the process. To demonstrate the efficacy of the proposed algorithm, we present simulations on a toy problem as well as on real world data sets to train ML models to perform classification tasks. We observe faster initial convergence and improved performance accuracy, especially in the low data range. We achieve on average 78\% and over 90\% classification accuracy respectively on the Gamma Telescope and mHealth data sets from the UCI ML repository. 
\end{abstract}


% \fi
\IEEEpeerreviewmaketitle

\section{Introduction}
%Distributed learning has received significant attention over the last decade as it posits a way to train a common ML model when the available data is distributed among multiple computing agents, without sharing raw data itself. Distributed algorithms for performing optimization and more recently sampling have been proposed to address these needs. Most conventional algorithms~\cite{parayil2020decentralized} are devised to be synchronously implemented by all agents in the sense that each agent updates its model parameters based on its local data and then takes a (weighted) average of its own and that of other agents (either all or some depending on the implementation of the communication channel). 

Distributed learning has received significant attention over the last decade as it posits a way to train a common ML model when the available data is distributed among multiple computing agents, without sharing raw data itself. Distributed algorithms for performing optimization and more recently Bayesian sampling have been proposed to address these needs. {Conventionally, algorithms, such as~\cite{parayil2020decentralized}, are devised to be synchronously implemented by all the agents such as each agent makes an update of its model parameter based on the (weighted) average of its own parameter and those of other agents. The entire process of  inter-agent communications and updates is assumed as a single event simultaneously carried out by all the agents. However, in practice this process constitutes an arbitrary sequence of parameter sharing and reception, and computational update in time~\cite{lamport2019time}. This means that the learning speed of the entire system is limited by the slowest agent if different computing hardware is used. Also, it is difficult for the agents to ascertain when all the agents finish their updates.}

%This means that the learning speed of the entire system is limited by the slowest agent if different computing hardware is used. 
To circumvent such issues, various asynchronous optimization strategies have been proposed~\cite{1104412, 5719290, 9217472}. Furthermore, individual agents are often made to perform \emph{multiple} intra-agent (or local) computations based on its local data before every inter-agent communication~\cite{tsitsiklis1984problems}. This prevents faster agents from sitting idle while the slower agents catch up. However, doing such local computations can lead to convergence to local optima or even completely lose convergence. Theoretical analysis for such situations is fairly limited, especially pertaining to Bayesian sampling methods.    

{In this paper, we propose an asynchronous distributed Bayesian learning framework which provides a fast and efficient way to learn under limited data constraints. The asynchrony is handled by the gossip protocol where the sequence of updates are random and  modeled as the Poisson's process. Each cycle of gossip entails only two active agents making updates. The proposed method employs multiple local computations via ULA with stochastic gradients per cycle. 

We provide theoretical analysis and  convergence guarantees of the proposed algorithm. We treat  the average of the samples generated by each agent as a random process. Both gossip and stochastic gradient introduce stochastic noise in the process. To handle such stochasticities, we employ the Fokker-Planck equation and establish that zero-mean stochastic noises do not affect convergence of the process as long as their second moment is bounded. Furthermore, using multiple computations per cycle is non-trivial to tackle as it can hamper convergence. We derive a bound for the total gradients in each cycle and thereby establish a criterion (controlled by the step size) to ensure that multiple computations do not affect  stability of the algorithm. Additionally, we assume a log-Sobolev inequality (LSI) on the posterior, which is less restrictive than the common log-concave assumption in the literature.} 

The algorithm is formulated such that the active agents can \emph{independently} perform their local computations in each cycle, allowing that the active agents communicate their samples (and some additional information) at the start of each cycle and then perform multiple (but the same number of) local computations independently without having to collaborate any more; at the end of their respective local computational cycle, the agents can simply rest. Since our algorithm is based on Bayesian sampling of the posterior, it is more resilient to over-fitting and the results have an uncertainty measure as well. 

Furthermore, in our approach, we assume a general communication topology: it can either occur over a graph as in decentralized learning or be coordinated by a central server as in federated learning. Our analysis shows bounds on the number of local computations that can be done by any agent to achieve consensus and convergence. Similar analysis for optimization counterparts exist, e.g.,~\cite{berahas2018balancing, nedic2018network}, but to the best of our knowledge, rigorous analysis for (distributed gossip-based) ULA is still lacking. Since the ULA algorithm has an inherent noise term dependent on the gradient step size, establishing bounds on convergence is non-trivial. 

%In this paper, we propose a Bayesian asynchronous distributed learning algorithm where agents use a gossip protocol to communicate over a graph. We quantify the effect of multiple local computations using Bayesian sampling based on ULA with stochastic gradients. In our approach, we assume a general communication topology: it can either occur over a graph as in decentralized learning or be coordinated by a central server as in federated learning. Our analysis shows bounds on the number of local computations that can be done by any agent to achieve consensus and convergence. Similar analysis for optimization counterparts exist, e.g.,~\cite{berahas2018balancing, nedic2018network}, but to the best of our knowledge, rigorous analysis for (distributed gossip-based) ULA is still lacking. Since the ULA algorithm has an inherent noise term dependent on the gradient step size, establishing bounds on convergence is non-trivial. 

The rest of the paper is organized as follows. We introduce the background concept in Section~\ref{sec:Preliminaries}. Thereafter, we present the proposed algorithm in Section~\ref{sec:Methodology}. The theoretical results and relevant discussion for the proposed algorithm are provided in Section~\ref{sec:Results}. Section~\ref{sec:Experiments}
presents the results of our numerical experiments accompanied with the observations and discussions from the empirical results. We conclude the paper in {Section~\ref{sec:Conclusions}}.

\textbf{Notation}: Let $\mathbb{R}^{n\times m}$ denote the set of $n\times m$ real matrices. For a vector $\bm{\phi}$, $\phi_i$ is the ${i^{th}}$ entry of $\bm{\phi}$. An $n\times n$ identity matrix is denoted as $I_n$. $\mathbf{1}_n$ denotes a $n$-dimensional vector of all ones and $\mathbf{e}_i$ is a $n$-dimensional vector with all $0$s except the $i$-th element being $1$. The $p$-norm of a vector $\mathbf{x}$ is denoted as $\left\| \mathbf{x} \right\|_p$ for $p \in [1,\infty]$. Given matrices $A \in \mathbb{R}^{m\times n}$ and $B \in \mathbb{R}^{p \times q}$, $A \otimes B \in \mathbb{R}^{mp \times nq}$ denotes their Kronecker product. For a graph $\mathcal{G}\left(\mathcal{V},\mathcal{E}\right)$ of order $n$, $\mathcal{V} \triangleq \left\{v_1, \ldots, v_n\right\}$ represents the agents and the communication links between the agents are represented as $\mathcal{E} \triangleq \left\{\varepsilon_1, \ldots, \varepsilon_{\ell}\right\} \subseteq \mathcal{V} \times \mathcal{V}$. Let $\mathcal{A} = \left[a_{i,j}\right]\in \mathbb{R}^{n\times n}$ be the \emph{adjacency matrix} with entries of $a_{i,j} = 1 $ if $(v_i,v_j)\in\mathcal{E}$ and zero otherwise. Denote by $\mathcal{N}_i$ the set of agents lined to $i$ or neighbors of the $i$-th agent. Define $\Delta = \text{diag} \left(\mathcal{A} \mathbf{1}_n \right)$ as the in-degree matrix and $\mathcal{L} = \Delta - \mathcal{A}$ as the graph \emph{Laplacian}. A Gaussian distribution with a mean $\mu\in\mathbb{R}^m$ and a covariance $\Sigma\in\mathbb{R}_{\geq0}^{m\times m}$ is denoted by $\mathcal{N}(\mu,\Sigma)$.

\input{Prelim}

\input{Methodology}

\input{Results}

\input{Experiments}

\input{Conclusions}

\bibliographystyle{IEEEtran}
\bibliography{Ref}
% \printbibliography

\input{Appendix}


%\vfill


\end{document}

%% file: Prelim.tex
\section{Preliminaries} \label{sec:Preliminaries}

Consider that $n$ agents hold a distributed collection of data sets $\{\bm{X}_i\}_{i=1}^{n}$ where $\bm{X}_i = \{x_i^j\}_{j=1}^{M_i}$ with $x_i^j \in \mathbb{R}^{d_w}$ for $i \in \mathcal{V}$. Each agent has access to only its own data set but not others' data sets and sharing of raw data is prohibited, i.e., $i$-th agent can access only $\bm{X}_i$. We denote the entire collection of data as $\bm{X} = \{\bm{X}_i\}_{i=1}^{n}$. Note that the notation $\bm{X}$ is for representation purposes only, no individual agent has access to the entirety of $\bm{X}$. Since individual agents do not have access to others' data sets, conventional centralized Bayesian sampling methods for training the machine learning models are not implementable. This situation calls for a collaborative learning paradigm where individual agents perform learning of unknown parameters based on their local data and then communicate their learned parameters (either directly over a graph topology or indirectly via a central server) with each other to obtain a (weighted) average as the final model parameters. 

The objective is to train a common model across all agents. Without loss of generality, we assume in the rest of the paper that the inter-agent communication occurs over a graph $\mathcal{G}(\mathcal{V},\mathcal{E})$. Note that indirect communication coordinated via a central server can be viewed as a special case of graph communication where the adjacency matrix reduces to $\mathcal{A} = \mathbf{1}_n \mathbf{1}_n^\top$ and all agents are connected to each other, i.e., $|\mathcal{N}_i| = n-1$ for all $i \in \mathcal{V}$.

\subsection{Bayesian inference} \label{sec:Bayesian_inf}
Bayesian learning provides a framework for \HB{learning} unknown parameters by sampling from a posterior distribution. The posterior distribution of interest here is the probability distribution of the unknown parameter $\bm{w}$ given the data $\bm{X}$, denoted by $p(\bm{w}|\bm{X})$. {We assume that each data point $x_i^j$, $\forall i\in \mathcal{V}$, $\forall j\in \{1, \ldots, M_i\}$) is independent {conditioned on $\bm{w}$}. %and identically distributed (IID), thus, each $\bm{X}_i$ is IID as well. 
Owing to this assumption, the target posterior distribution $p^*(\bm{w}) \triangleq p(\bm{w}|\bm{X})$ can be expressed as}
\begin{align}
    &p(\bm{w}|\bm{X}) \propto p(\bm{w}) \prod_{i=1}^{n} p(\bm{X}_i|\bm{w}) = \prod_{i=1}^{n} p(\bm{X}_i|\bm{w}) p(\bm{w})^{\frac{1}{n}}. \label{eq:global_posterior}
\end{align}

The objective of the Bayesian inference problem is to draw samples of $\bm{w}$ from $p^*$. Since analytical solutions to $p^*$ are intractable for most practically relevant problems, indirect methods like \emph{Markov Chain Monte Carlo} (MCMC) or \emph{variational inference} are frequently employed. For the purpose of this paper, we use MCMC for sampling $\bm{w}$ from the desired posterior $p^*$. In MCMC a single chain is initiated from (either deterministic or random) initial state and made to evolve based on the present state until its underlying distribution reaches the target distribution. Likewise, multiple chains are drawn and inference is performed based on the final states of all the chains. %Higher the number of chains used, greater is the accuracy of the model. 
There are various ways to evolve the MCMC states. We make use of \emph{ULA} which is discussed next.

\subsection{Unadjusted Langevin Algorithm (ULA)} \label{sec:Langevin_algo}
%As discussed in the preceding section, for MCMC sampling we need a strategy to evolve the samples such that their distribution effectively converges to the desired posterior. To this end, we use 
The unadjusted Langevin algorithm (ULA) is a discretized algorithm based on the continuous-time Langevin dynamics. Both the Langevin dynamics and the ULA {(with properly chosen diminishing step-sizes)} are known to converge \emph{exponentially} fast to the target distribution. %\HB{ULA with properly chosen diminishing step-sizes?}  

We rewrite the target posterior in terms of an \emph{energy function} defined as 
\begin{align}
    p(\bm{w}|\bm{X}) \propto \exp(-E(\bm{w})).\label{eq:energy_fun_def}\end{align} 
The continuous-time Langevin dynamics is a stochastic differential equation (SDE) given by
\begin{align}
    d\bm{w}(t) &= -\nabla E(\bm{w})dt + \sqrt{2} d\bm{B}(t), \label{eq:cont_Lang_dyn}
\end{align}
where $\bm{B}(t)$ is a $d_w$-dimensional Brownian motion. Euler discretization of~\eqref{eq:cont_Lang_dyn} leads to the canonical ULA for centralized data, which takes the following form
\begin{align}
    &\bm{w}_{k+1} = \bm{w}_{k} - \alpha_k \nabla E(\bm{w}_k) + \sqrt{2 \alpha_k} \bm{v}_k, \label{eq:canonical_ULA}
\end{align}
where $\alpha_k>0$ is a (time-dependent) gradient step size, $\bm{v}_{\HB{k}} \sim \mathcal{N}(\mathbf{0}_{d_w}, I_{d_w})$ and $k$ refers to the discrete time indices of the discretization. In essence, ULA merges gradient information $\nabla E(\bm{w})$ with a random Gaussian perturbation $\bm{v}$ (centered at the present state and with a variance dependent on $\alpha_k$) to better explore the sampling space and prevent over-fitting of data and convergence to local minima of $E$. Note that when the algorithm takes bigger gradient leaps, the perturbation is more diffused around the present state and vise-versa. 

A vanilla distributed version of ULA is presented below.
\begin{align}
\begin{split}
    &\bm{w}_{i,k+1} = \bm{w}_{i,k} - \beta_k \sum_{j \in \mathcal{N}_i} \big(\bm{w}_{i,k} - \bm{w}_{j,k}\big) \\
    &\quad - \alpha_k n \nabla E_i(\bm{w}_{i,k}) + \sqrt{2\alpha_kn} \, \bm{v}_{i,k}, \label{eq:canonical_DULA}
\end{split}
\end{align}  
where $\bm{w}_{i,k}$ is the sample of the $i$-th agent at time index $k$, $\beta_k>0$ is (potentially time-dependent) fusion coefficient, $\bm{v}_{i,k} \sim \mathcal{N}(\mathbf{0}_{d_w}, I_{d_w})$, and $\nabla E_i(\bm{w}_{i,k}) = - \nabla \log p(\bm{X}_i|\bm{w}_{i,k}) - \frac{1}{n} \nabla \log p(\bm{w}_{i,k}) = - \sum_{j=1}^{M_i} \nabla \log p(x_i^j|\bm{w}_{i,k}) - \frac{1}{n} \nabla \log p(\bm{w}_{i,k})$ is the local gradient computed by $i$-th agent based on $\bm{X}_i$. In~\eqref{eq:canonical_DULA}, no single entity is needed to access the entire data set $\bm{X}$. Individual agents compute their local gradients $\nabla E_i(\bm{w}_{i,k})$ based on their local data $\bm{X}_i$ and fuse their sample parameters via the term $\beta_k \sum_{j \in \mathcal{N}_i} (\bm{w}_{i,k} - \bm{w}_{j,k})$. Detailed theoretical guarantees for~\eqref{eq:canonical_DULA} have been established in~\cite{parayil2020decentralized}. Under certain assumptions, the convergence rate is shown to be \emph{polynomial}. The decline in rate from exponential (for centralized ULA) to polynomial (for distributed ULA) is a price paid for the distributed implementation.    

\subsection{Gossip-based Protocol} \label{sec:Goss_protocol}
Although~\eqref{eq:canonical_DULA} proposed by~\cite{parayil2020decentralized} tackles the problem of distributed Bayesian learning, it needs synchronous data fusion and gradient update by all agents. This may not be feasible to implement depending on the quality of hardware available. In most practical cases, some discrepancy in synchronous communication are bound to occur. A rather simple way to model asynchrony in the communication over the graph is to assume a gossip protocol. 

To describe the gossip communication, we first introduce the notions of \emph{local clock} and \emph{universal clock}. It is assumed that every agent has a local clock which ticks at rate $1$ of a Poisson process. The ticks of each agent's local clock are independent of each other and no two local clocks tick at the same time. Thus, the ticks of all agents combined (denoted by $k$) corresponds to a \emph{universal clock} which ticks at rate $n$ Poisson process. In gossip, a agent becomes active at the tick of its own local clock (denoted as $i_k$), the probability of which is \emph{uniform} among all $n$ agents. The active agent then uniformly randomly chooses another agent among its neighbors (denoted as $j_k$) to communicate with. Together, these two agents make updates and share their samples for averaging, while all others remain dormant. The probability of any agent $i$ to be active at $k$-th tick of the universal clock is given by $p_i = \frac{1}{n}\left(1+ \displaystyle \sum_{j\in \mathcal{N}_i} \frac{1}{|\mathcal{N}_j|} \right)$.

Denote by $\mathcal{A}_k = \{i_k,j_k\}$ the set of two  active agents at the $k$-th tick of the universal clock and $\tau_i(k)$ the number of times agent $i$ has been active until the $k$-th tick of the universal clock. The update algorithm for the active agents, i.e., $i \in \mathcal{A}_k$, is given by 
\begin{align}
    &\bm{w}_i(\tau_i(k)+1) = \bm{w}_i(\tau_i(k)) - \beta \sum_{j \in \mathcal{A}_k} \left( \bm{w}_i(\tau_i(k)) - \bm{w}_j(\tau_j (k))\right) \nonumber \\
    &\qquad -\frac{n\alpha_k}{p_i} \nabla E_i( \bm{w}_i(\tau_i(k))) + \sqrt{2\alpha_k} \bm{v}_i(\tau_i(k)), \label{eq:ULA_gossip_local}
\end{align}
where $\bm{w}_i(\cdot)$ represents the value at the corresponding local clock index and $\bm{v}_i \sim \mathcal{N}\left( \mathbf{0}_{d_w}, \frac{n^2}{2} I_{d_w} \right)$. We represent ~\eqref{eq:ULA_gossip_local} in the universal clock ticks as
\begin{align}
\begin{split}
    &\bm{w}_{i,k+1} = \bm{w}_{i,k} - \delta_{i,k} \Bigg[\beta \sum_{j \in \mathcal{A}_k} \Big( \bm{w}_{i,k} - \bm{w}_{j,k} \Big) \\
    &\qquad -\frac{n\alpha_k}{p_i} \nabla E_i( \bm{w}_{i,k}) + \sqrt{2\alpha_k} \bm{v}_{i,k} \Bigg], \label{eq:ULA_gossip_local_2}
\end{split}
\end{align}
where the indicator function $\delta_{i,k}$ is defined as 
\begin{align} \label{eq:indicator_fn}
    \delta_{i,k} = 
    \begin{cases}
        1, \quad i \in \mathcal{A}_k \\
        0, \quad i \notin \mathcal{A}_k.
    \end{cases}
\end{align} 
Thus, in~\eqref{eq:ULA_gossip_local}, if $i$ is active at time index $k$, it makes an update accordingly, otherwise, it retains its old value $\bm{w}_{i,k}$. Note that each agent is only aware of its own local clock, hence knows its own $\tau$ value. However, the ticks of the universal clock $k$ is unknown to all agents. The concept of the universal clock has been introduced for analysis purposes only. 

\subsection{Stochastic Gradient} \label{sec:sto_grad}
A well-established tool to speed-up a gradient-based algorithm is to use \emph{stochastic gradient}. In this method, instead of using the entire batch of data to compute the gradient, only a fraction, say $f\times 100\%$, $0<f\leq 1$, of the entire data (\emph{mini-batch}), which is \emph{randomly} chosen from a \emph{uniform} distribution each time, can be used to compute an estimate of the true gradient. Since the computation complexity of the gradient scales roughly as $\mathcal{O}(M)$ (where $M$ is the total training data available), using only $fM$ amount of the training data reduces gradient computation complexity to $\mathcal{O}(fM)$. %\HB{Define $f$?}

%A caveat here is to ``randomly'' choose the mini-batch (with uniform probability) that is to be used for the stochastic gradient computation at every iteration. The idea is that over the course of a long run, the stochasticity in the gradient computation averages out, i.e., the expectation of the stochastic gradient equals the true gradient. 
We adapt the distributed ULA in~\eqref{eq:canonical_DULA} for stochastic gradient and obtain 
\begin{align}
\begin{split}
    &\bm{w}_{i,k+1} = \bm{w}_{i,k} - \beta_k \sum_{j \in \mathcal{N}_i} \big(\bm{w}_{i,k} - \bm{w}_{j,k}\big) \\
    &\qquad - \alpha_k n \widehat{\nabla E}_i(\bm{w}_{i,k}) + \sqrt{2\alpha_kn} \, \bm{v}_{i,k}, \label{eq:canonical_DULA_sto_grad}
\end{split}
\end{align}
where the stochastic gradient is given as $\widehat{ \nabla E}_i(\bm{w}_{i,k}) = - \frac{1}{n} \nabla \log p(\bm{w}_{i,k})  - \frac{1}{f_i} \nabla \log p(\hat{\bm{X}}_i|\bm{w}_{i,k}) = - \frac{1}{n} \nabla \log p(\bm{w}_{i,k}) - \frac{1}{f_i} \displaystyle \sum_{j=1}^{m_i} \nabla \log p(x_i^j|\bm{w}_{i,k})$, in which $\hat{\bm{X}}_i = \{x_i^j\}_{j=1}^{m_i} \subset \bm{X}_i$ with $|\hat{\bm{X}}_i|=m_i$ is the mini-batch data set randomly used by the $i$-th agent for computing the stochastic gradient, and $f_i = \frac{m_i}{M_i}$. Note, from implementation perspective, there is almost no difference between~\eqref{eq:canonical_DULA} and~\eqref{eq:canonical_DULA_sto_grad}, except that the gradient calculated in~\eqref{eq:canonical_DULA_sto_grad} the agents use randomly selected fraction of their respective data sets. Thus, stochastic gradient provides computational cost savings at no added implementation difficulty. 

%As we shall see that our proposed algorithm uses multiple local computations per cycle, each involving a gradient calculation. Using stochastic gradient instead of the true gradient is a great tool to speedup the process. 

%\HB{(CHANGE) In our simulations, while using stochastic gradient, each agent is made to use the same mini-batch for all the gradient calculations within the same cycle. This may be helpful as in some situations loading the data from the storage to memory may cause an additional overhead. In this case, once the mini-batch is loaded to memory for computation, it is reused for all the local computations within that cycle. At the next cycle, a different mini-batch is loaded and re-used until the end of the cycle. }

% However,~\eqref{eq:canonical_DULA} assumes that each agent at every discrete time step $k$ makes a gradient update and fuses the parameters from the neighboring agents as well in a synchronous fashion. This necessitates frequent communication which is not always pragmatic. In practice, individual agents often perform multiple updates based on their local data between successive communications to perform the fusion step. Although some convergence analysis exists in literature for optimization under such situations, similar analysis is largely lacking for Bayesian sampling counterparts.

%% file: Methodology.tex
\section{Methodology} \label{sec:Methodology}
In this paper, we focus on obtaining theoretical convergence guarantees for distributed Bayesian algorithm where nodes communicate via a gossip protocol over a graph and make multiple local intra-node ULA computations between successive inter-node communication. We next introduce the proposed algorithm. 

Assume that at $k$-th tick of the universal clock, $\mathcal{A}_k = \{i_k, j_k\}$ are active. The parameter update for $i \in \mathcal{A}_k$ consists of the following steps:
\begin{align}
    \bm{w}_{i,k}^0 &= \bm{w}_{i,k} - \beta \sum_{j \in \mathcal{A}_k} \Big(\bm{w}_{i,k} - \bm{w}_{j,k} \Big), \label{eq:fusion} \\
    \begin{split}
    \bm{w}_{i,k}^1 &= \bm{w}_{i,k}^0 - \frac{n\alpha_k}{p_i} \widehat{\nabla E}_i(\bm{w}_{i,k}^0) + \sqrt{ \alpha_kn^2}\,\, \bm{v}_{i,k}^0, \\
    \bm{w}_{i,k}^2 &= \bm{w}_{i,k}^1 - \frac{n\alpha_k}{p_i} \widehat{\nabla E}_i(\bm{w}_{i,k}^1) + \sqrt{\alpha_kn^2}\,\, \bm{v}_{i,k}^1, \\
    \vdots \\
    \bm{w}_{i,k}^{T} &= \bm{w}_{i,k}^{T-1} - \frac{n\alpha_k}{p_i} \widehat{\nabla E}_i(\bm{w}_{i,k}^{T-1}) + \sqrt{\alpha_kn^2}\,\, \bm{v}_{i,k}^{T-1}, 
    \end{split} \label{eq:mul_comp} \\
    \bm{w}_{i,k+1} &\equiv \bm{w}_{i,k}^{T}. \label{eq:final_update}
\end{align}
{In the above notation of $\bm{w}_{i,k}^j$, $i$ denotes the $i$-th node, $k$ denotes the ticks of the universal clock (pertaining to a single cycle), and $j$ is the index for local intra-node computations.} $T$ is the number of local computations done by the active nodes $i \in \mathcal{A}_k$ in the $k$-th cycle. $T$ can potentially be different at every $k$, but for simplicity of notation, we assume a uniform $T$ across all the agents. Note that within a single cycle $T$ is the same for both the gossiping node pair. The stochastic gradients in all the computations in~\eqref{eq:mul_comp} are given as $\widehat{\nabla E}_i(\bm{w}_{i,k}^j) = -\frac{1}{n} \nabla \log p(\bm{w}_{i,k}^j) - \frac{1}{f_i} \nabla \log p(\hat{\bm{X}}_i| \bm{w}_{i,k}^j)$ ($0\leq \ell \leq T$) where $\hat{\bm{X}}_i$ is the mini-batch data set randomly chosen by $i$-th node to compute the gradient and $\hat{\bm{X}}_i$ remains the same for all the multiple local computations performed in the cycle. This is beneficial where a large amount of data can impose an overhead in loading data from the agent's storage to memory. Since the mini-batch needs to be loaded to memory at the start of the cycle and thereafter the same mini-batch is re-used, it avoids repetitive loading overhead within the cycle. However, a new mini-batch is loaded in a different cycle. Also, recall that, following the gossip protocol, agent $i$, $\forall i\notin \mathcal{A}_k$ does not make any updates meanwhile. 

The process in~\eqref{eq:fusion}--~\eqref{eq:final_update} represents a single \emph{cycle} at the $k$-th tick. The cycle begins with a linear inter-node fusion represented in~\eqref{eq:fusion}, followed by a series of multiple ($T\geq1$) local intra-node computations represented by~\eqref{eq:mul_comp}, and finally the value at the end of the last local computation is the updated sample for the $(k+1)$-th tick as shown by~\eqref{eq:final_update}. %\HB{Do we need this introduction of notation here? In the above notation of $\bm{w}_{i,k}^j$, $i$ denotes the $i$-th node, $k$ denotes the ticks of the universal clock (pertaining to a single cycle) and $j$ index for local intra-node computations.} 
Within each cycle, the first fusion step needs to be synchronized between both the gossiping nodes, however, the rest of the local computations can be asynchronously performed at the node level. No additional communication between the nodes are needed at the end of the local computations when the cycle ends. The number of local computations ($T$) can either be pre-determined at the start of the training process and the same for all nodes across all cycles (\emph{static mode}), or they can be determined by the gossiping nodes on the fly (\emph{dynamic mode}) at the start of every cycle (such that it is the same for both gossiping nodes). In the dynamic mode, each gossiping node can determine a value and share it with the other along with its sample value, and then some pre-determined protocol can be used by each to finalize the value, for example, choosing $T$ as $\floor*{\frac{T_{i_k}+ T_{j_k}}{2}}$.

Another essential aspect of the local computations is the gradient step sizes used. To achieve asymptotic convergence, decreasing step sizes are used which are based on the number of (synchronous) iterations. However, since gossip is an asynchronous communication protocol, each node has a different iteration (i.e., ticks of the corresponding local clock). Also the tick of the universal clock is unknown to all the nodes. 
Thus, the step sizes are designed based on the ticks of the local clocks of the active nodes. While sharing the sample values to perform~\eqref{eq:fusion}, the gossiping nodes need to share their respective $\tau$ values as well. Thereafter, a common step size $\alpha_k$ is computed as below. 
\begin{align}
    \alpha_k &= \frac{a}{\Big(\min\{\tau_{i_k}, \tau_{j_k}\}+1 \Big)^{\delta_\alpha}}, \label{eq:alpha_def}
\end{align}
where $a>0$ %(and satisfies Conditions~\ref{cond:1} and~\ref{cond:2}) 
and $\delta_\alpha>0$ are user-defined hyperparameters shared by all nodes. Note that $\alpha_k$ remains fixed for  all the computations in the cycle, but eventually decreases as the cycles progress. As we shall later see, this decreasing step size is needed to ensure asymptotic convergence.

{In the rest of the paper, we refer to the case where $T_i=1$ for all $i\in \mathcal{V}$ as the \emph{canonical} distributed ULA. Thus, in the canonical distributed ULA, every node performs exactly ``one'' update and simultaneously communicates its parameter values to others in a synchronous fashion. We shall use the canonical case as the baseline to compare the results of our experiments in Section~\ref{sec:Experiments}.}

\begin{comment}
\begin{figure}
  \centering
  \includegraphics[width=\linewidth,height=4cm]{Figures/cycle.png}
  \caption{Illustration of the synchronous and asynchronous phases in a cycle.}
  \label{fig:cycle}
\end{figure}
\end{comment}

\setlength{\textfloatsep}{0.5cm}
\setlength{\floatsep}{0.5cm}
\begin{algorithm}[t]
 \caption{Distributed Gossip ULA with multiple asynchronous local computations}
 \label{Algorithm1}
 \begin{algorithmic}[1]
  \STATE \textit{Input} : $a$, $\beta$, $\delta_\alpha$ and $T$ (only for static mode)
  \STATE \textit{Initialization} : $\mathbf{w}_0 = \begin{bmatrix} {\bm{w}}_{1,0}^\top & \ldots & {\bm{w}}_{n,0}^\top \end{bmatrix}^\top$
  \FOR {$k \geq 0$}
  \STATE \textit{Some node $i_k$ becomes active}
  \STATE \textit{$i_k$ randomly selects $j_k$; $\mathcal{A}_k = \{i_k,j_k\}$}
  \FOR {$i = 1$ to $n$}
  \IF {$i \in \mathcal{A}_k$}
  \STATE \textit{Determine} $T_i$ \textit{(if dynamic mode)}
  \STATE \textit{Broadcast $\bm{w}_{i,k-1}$, $\tau_i$ and $T_i$ (for dynamic mode)}
  \STATE \textit{Receive $\bm{w}_{j,k-1}$, $\tau_j$ and $T_j$ (for dynamic mode), where $j \in \mathcal{A}_k/\{i\}$}
  \STATE \textit{Compute} $\bm{w}_{i,k}^0 = \bm{w}_{i,k-1} - \beta \sum_{j\in \mathcal{A}_k} \left(\bm{w}_{i,k-1} - \bm{w}_{j,k-1}\right)$.
  \STATE \textit{Calculate  $T= \floor*{\frac{T_{i_k}+ T_{j_k}}{2}}$ (for dynamic mode)}
  \STATE \textit{Calculate } $\alpha_k = \frac{a}{(\min\{\tau_{i_k}, \tau_{j_k}\} +1)^{\delta_\alpha}}$
  \FOR {$j = 1$ to $T$}
  \STATE \textit{Sample } $\bm{v}_{i,k}^j \sim \mathcal{N}(\mathbf{0}_{d_w}, I_{d_w})$  
  \STATE \textit{Compute } $\nabla E_i(\bm{w}_{i,k}^j)$ \label{Step041}
  \STATE \textit{Update }  $\bm{w}_{i,k}^{j} = \bm{w}_{i,k}^{j-1}  - \frac{n\alpha_k}{p_i} \nabla E_i (\bm{w}_{i,k}^{j-1}) + \sqrt{\alpha_k n^2} \bm{v}_{i,k}^{j-1}$
  \ENDFOR
  \STATE \textit{Assign } $\bm{w}_{i,k+1} = \bm{w}_{i,k}^T$
  \ELSIF{$i \notin \mathcal{A}_k$} 
  \STATE \textit{No update } $\bm{w}_{i,k+1} = \bm{w}_{i,k}$
  \ENDIF
  \ENDFOR
  \ENDFOR
 \end{algorithmic}
 \end{algorithm}

%% file: Results.tex
\section{Theoretical results} \label{sec:Results}
In this section, we present the theoretical results and a brief outline of the corresponding analysis. We begin by stating our assumptions below.
\begin{assumption} \label{assump:Lipz}
All the individual gradients $\nabla E_i(\cdot) $ are Lipschitz continuous for all $i \in \mathcal{V}$, i.e., for any $\bm{w}_1, \bm{w}_2 \in \mathbb{R}^{d_w}$, we have
\begin{align} \label{eq:Lipz_1}
    \|\nabla E_i(\bm{w}_1) - \nabla E_i(\bm{w}_2)\| &\leq L_i \|\bm{w}_1 - \bm{w}_2\|.
\end{align}
\end{assumption}

{From~\eqref{eq:Lipz_1}, there exists $\bar{L}>0$ such that
\begin{align}
    \|\nabla E(\bm{w}_1) - \nabla E(\bm{w}_2)\| &\leq \bar{L} \|\bm{w}_1 - \bm{w}_2\|,
\end{align}
where $\nabla E(\cdot)$ is defined in~\eqref{eq:grad_terms_1}. Also, define $L = \displaystyle \max_{i\in \{1, \ldots, n\}} L_i$.}

\begin{assumption} \label{assump:Graph}
The interaction topology of \HB{$n$} networked nodes is given as a connected undirected
graph $\mathcal{G}(\mathcal{V}, \mathcal{E})$.
\end{assumption}

For a connected undirected graph $\mathcal{G}\left(\mathcal{V},\mathcal{E}\right)$, the graph Laplacian $\mathcal{L}$ is a positive semi-definite matrix with one eigenvalue at 0 corresponding to the eigenvector $\mathbf{1}_n$. Furthermore, it follows from Lemma~$3$ in~\cite{Gutman04} that for all $\mathbf{x}\in\mathbb{R}^n$, such that $\mathbf{1}^T_n\mathbf{x} = 0$, we have $\mathbf{x}^T\mathcal{L}\left( \mathcal{L} \right)^+\mathbf{x} =  \mathbf{x}^T\mathbf{x}$, {where $(\cdot)^+$ denoted the pseudo-inverse}. 

Let $\mathcal{F}_{k}$ denotes a filtration generated by the sequence $\{\mathbf{w}_{0},\ldots,\mathbf{w}_{k}\}$, i.e., $\mathbb{E}[\,\mathbf{v}{(k)}\,|\mathcal{F}_{k}] = \mathbf{0}$. 

%\HB{\begin{assumption}\label{assump:bounded_comp}
%The number of local computations $T$ is finite and bounded. Specifically, $T$ must satisfy~\eqref{eq:cond_1} in Conditions~\ref{cond:1} and~\ref{cond:3}.
%\end{assumption}}

\begin{assumption} \label{assump:sto_grad}
Let $\xi_{i,k}^j = \nabla E(\bar{\bm{w}}_k^j) - \sum_{i\in \mathcal{A}_k} \frac{1}{p_i} \nabla E_i(\bar{\bm{w}}_k^j)$ and $\zeta_{i,k}^j = \widehat{\nabla E}(\bm{w}_{i,k}^j) - \nabla E(\bm{w}_{i,k}^j)$, then for all $i \in \mathcal{V}$ the second moment of $\xi_{i,k}^j$ and $\zeta_{i,k}^j$ are bounded, i.e., 
\begin{align} \label{eq:assump_3_1}
    \mathbb{E}[\|\xi_{i,k}^j\|^2] &\leq C_\xi, \quad \forall \, 0 \leq k, 0\leq j < T, 
\end{align}
and 
\begin{align} \label{eq:assump_3_2}
    \mathbb{E}[\|\zeta_{i,k}^j\|^2] &\leq C_\zeta, \quad \forall \, 0 \leq k, 0\leq j < T.
\end{align}
\end{assumption}

The terms $\zeta_{i,k}^j$ and $\xi_{i,k}^j$ stem from the randomness in the gradient of the mean of the samples introduced by the mini-batch stochastic gradient and the gossip protocol respectively. This is expressed in detail in~\eqref{eq:grad_terms}. $\zeta_{i,k}^j$ and $\xi_{i,k}^j$ are related to the terms defined in~\eqref{eq:grad_terms_2} and~\eqref{eq:grad_terms_4}.

\begin{assumption} \label{assump:LSI}
The posterior distribution $p^*$ satisfies the log-Sobolev inequality (LSI) condition, i.e., for any function $g(\bar{\bm{w}})$ with $\mathbb{E}_{p^*}[ g(\bar{\bm{w}})] = 1$, there exists a constant $\rho_U>0$ such that the following condition is satisfied.
\begin{align} \label{eq:LSI_original}
    \mathbb{E}_{p^*} \big[g(\bar{\bm{w}}) \log g(\bar{\bm{w}}) \big] \leq \frac{1}{2\rho_U} \mathbb{E}_{p^*} \left[ \frac{\|\nabla g(\bar{ \bm{w}})\|^2}{g(\bar{\bm{w}})} \right].
\end{align}
\end{assumption}
%If $g({\bar{\bm{w}}}) = \displaystyle \frac{p_t({\bar{\bm{w}}})}{p^*({\bar{\bm{w}}})}$, the inequality~\eqref{eq:LSI_original} becomes
%\begin{align}
%\begin{split}
%    F(p_t(\bar{\bm{w}})) 
%    \, & \triangleq \, \mathbb{E}_{p_{t}(\bar{\bm{w}})}  \left[  \log\left(\frac{p_t(\bar{\bm{w}})}{p^*}\right) \right] \\
%    & \leq \, \frac{1}{2 \rho_U} \mathbb{E}_{p_{t}(\bar{\bm{w}})}  \left[ \left\|\nabla \log\left(\frac{p_t(\bar{\bm{w}})}{p^*}\right)\right\|_2^2 \right]. \label{eq:LSI}
%\end{split}
%\end{align}

All the aforementioned assumptions are commonly used in related literature. Assumption~\ref{assump:Lipz} implies smoothness of the log of the posterior function while the LSI Assumption~\ref{assump:LSI} is needed for the convergence of the distribution. The LSI assumption is weaker than log-concave assumption used in some existing literature~\cite{dalalyan2017theoretical, dalalyan2017further, cheng2018underdamped, cheng2018convergence, durmus2016sampling, durmus2017nonasymptotic, durmus2019high}. Assumption~\ref{assump:Graph} implies sufficient connectivity of the graph to ensure consensus and is trivial in the case of federated learning where a central server accesses all nodes' parameters for fusion. Assumption~\ref{assump:sto_grad} is common in stochastic gradient literature and holds true experimentally.  %In light of the stated assumptions, we now proceed to state our results.

%%%%%%%%%%%%%%%%%%%%%%%%%%%%%%%%%%%%%%%%%%%%%%%%%%%%%%%%%%%%%%%%%%%%%%%

\subsection{Consensus}
Define $\mathbf{w}_k = [\bm{w}_1^\top, \ldots, \bm{w}_n^\top]^\top \in \mathbb{R}^{nd_w \times nd_w}$ and $\bar{\bm{w}}_k = \frac{1}{n} \sum_{i=1}^{n} \bm{w}_k$. We quantify the consensus error $\tilde{\mathbf{w}}_k$ by the difference of each individual sample from their average, i.e.,
\begin{align}
    \tilde{\mathbf{w}}_k \triangleq \mathbf{w}_k - (\mathbf{1}_n \otimes \bar{\bm{w}}_k) = \left(I_{nd_w} - \frac{1}{n} \mathbf{1}_n \mathbf{1}_n \otimes I_{d_w} \right) \mathbf{w}_k
\end{align}

{The Laplacian matrix corresponding to the active gossiping agents $\mathcal{L}_k$ is given by $\mathcal{L}_k = (\mathbf{e}_{i_k} - \mathbf{e}_{j_k})(\mathbf{e}_{i_k} - \mathbf{e}_{j_k})^\top  \in \mathbb{R}^{n \times n}$. Let $\bar{\mathcal{L}} = \mathbb{E}[\mathcal{L}_k]$ and $\mathbf{w}^* = \left[ \bm{w}_1^*, \ldots, \bm{w}_n^* \right]^\top$ where $\bm{w}_i^*$ corresponds to some local minima of $E_i(\cdot)$ (i.e., $\nabla E_i(\bm{w}_i^*) = 0$).}

For asymptotic convergence to be achieved, we need $\mathbb{E} \|\tilde{\mathbf{w}}_k\|^2 \to 0$ as $k\to \infty$. Our analysis shows that the propagation of the consensus error from each subsequent cycle follows the following rate:
\begin{align}
    \mathbb{E}[\| \tilde{\mathbf{w}}_{k+1} \|^2] &\leq \bar{\lambda} \mathbb{E}[\| \tilde{\mathbf{w}}_{k}\|^2] + \frac{C_\mu}{(k+1)^{\delta_\alpha}}, \label{eq:con_err_syn_norm_sq_f}
\end{align}
where 
\begin{align}
    \bar{\lambda} &= \lambda(\theta+1)^2 \left[1 + \frac{5a^2n^2L^2T^2}{\theta p_m^2(1 - anLT/p_m)^2} \right] \qquad \qquad  
    \text{and} \\
\begin{split}
    C_\mu &= \frac{5a^2n^2T^2(\theta+1)^2}{\theta(1-anLT/p_m)^2} \left[2C_\zeta + L^2 C_{_{\mathbf{w}^*}} + 2L^2 C_{_{\bar{\bm{w}}}} \right] \\ 
    &\quad + \frac{4anTd_w(\theta+1)}{\theta} \left[1 + \frac{5a^2n^2L^2T^2(\theta+1)}{ p_m^2 (1-anLT/p_m)^2}\right]
\end{split}   
\end{align}
{in which $\lambda$ is the second largest eigenvalue of $\bar{\mathcal{L}}$, $\,\mathbb{E}[ \|\mathbf{w}^*\|^2] \leq C_{_{\mathbf{w}^*}}$, $p_m = \min_{i\in \mathcal{V}}\{p_i\} < \frac{1}{n}$} and $\theta>0$ is an arbitrarily chosen parameter needed for analysis. Before introducing our result on the rate of consensus, we need the following conditions to be satisfied. 
\begin{condition} \label{cond:1}
\begin{align}
    aT < \frac{p_m}{nL}. \label{eq:cond_1} 
\end{align}
\end{condition}
Condition~\ref{cond:1} arises from the proof of bounded gradient required for consensus.  It can always be satisfied by adjusting $a$ and $T$. Furthermore,~\eqref{eq:cond_1} shows that to safely perform more local computations, a smaller initial step size $a$ is advisable. 

\begin{condition} \label{cond:2}
\begin{align}
    \bar{\lambda} < 1. \label{eq:cond_2}
\end{align} 
\end{condition}
Condition~\ref{cond:2} is needed to ensure that the consensus error decreases with every cycle performed. It can be shown that there always exists some $a \in (0,1)$ that  ensures~\eqref{eq:cond_2} as long as $\lambda < \frac{1}{(\theta+1)^2}$. Since $\theta>0$ is arbitrarily chosen, 
for any $\lambda$ (which is a graph parameter and cannot be controlled), a large enough $\theta$ can be found to satisfy $\lambda < \frac{1}{(\theta+1)^2}$. Hence, a small enough $a\in (0,1)$ ensures~\eqref{eq:cond_2}.

\begin{condition} \label{cond:3}
\begin{align}
    {a < \left( \frac{\rho_U(3\delta_\alpha-1)}{16 T\bar{L}^4(6\delta_\alpha -1)} \right)^{\frac{1}{3}}.} \label{eq:cond_3} 
\end{align}
\end{condition}
Condition~\ref{cond:3} is a product of the bounded second moment of the average of all the samples generated by all the agents. Refer to the proof of Lemma~\ref{lemma:bounded_avg_samples} for details. For chosen values of $\delta_\alpha$ and $T$, there will always exist some $a\in (0,1)$ such that~\eqref{eq:cond_3} is satisfied. 

We now state the rate of consensus for the proposed algorithm.
\begin{theorem} \label{thm:consensus}
Suppose that Assumptions~\ref{assump:Lipz}--\ref{assump:sto_grad} hold and Conditions~\ref{cond:1}-\ref{cond:3} are satisfied, then the consensus error $\tilde{\mathbf{w}}(k+1)$ satisfies
\begin{align}
    \mathbb{E}[\|\tilde{\mathbf{w}}_{k+1}\|^2] &\leq W_1 \bar{\lambda}^{k+1} + \frac{W_2}{(k+1)^{\delta_\alpha}}, \label{eq:consensus_rate}
\end{align}
where 
\begin{align}
    W_1 &= \mathbb{E}[\|\tilde{\mathbf{w}}_0 \|^2] + C_\mu \sum_{t=0}^{\bar{t}-1} \frac{\bar{\lambda}^{-(t+1)}}{(t+1)^{\delta_\alpha}}, \\
    W_2 &= - \frac{C_\mu}{\bar{\lambda}} \left( \ln{\bar{\lambda}} + \frac{\delta_\alpha}{\bar{t}+1} \right)^{-1}, \qquad \qquad  
    \text{and} \\
    \bar{t} &= \max\Bigg\{0,\ceil[\Bigg]{\frac{\delta_\alpha}{|\ln{\bar{\lambda}}|} - 1} \Bigg\}.
\end{align}
\end{theorem}

Equation~\eqref{eq:consensus_rate} in Theorem~\ref{thm:consensus} shows that the consensus error, under the constraints in~\eqref{eq:cond_1} and~\eqref{eq:cond_2}, asymptotically vanishes. The rate of consensus $\mathcal{O} \left(k^{-\delta_\alpha} \right)$ is the same as that in the case of canonical distributed ULA up to some constants. 

\subsection{Convergence}
In this section, we look at the convergence of the proposed algorithm in a probabilistic sense. To quantify the convergence, we use the Kullback-Leibler (KL) divergence. Let the distribution of the average of the samples generated by all agents at end of each cycle be denoted by $p(\bar{\bm{w}})$) and the target distribution $p^*$. The KL divergence of $p(\bar{\bm{w}})$ from $p^*$, denoted by $F\left( p( \bar{\bm{w}})\right)$, is defined as below.
\begin{align}
    F\left(p(\bar{\bm{w}})\right) &= \int p(\bar{\bm{w}}) \log \left( \frac{p(\bar{\bm{w}})}{p^*}\right) d\bar{\bm{w}}. \label{eq:KL_div_def}
\end{align}
The KL divergence is a non-negative function which is only zero when $p(\bar{\bm{w}}) \equiv p^*$. Thus, the asymptotic convergence of $F \left( p(\bar{\bm{w}})\right) \to 0$ implies $p(\bar{\bm{w}}) \to p^*$. {If $g({\bar{\bm{w}}}) = \displaystyle \frac{p({\bar{\bm{w}}})}{p^*}$ in~
\eqref{eq:LSI_original}, we obtain the following bound on $F\left( p( \bar{\bm{w}})\right)$ 
\begin{align}
\begin{split}
    F(p(\bar{\bm{w}})) 
    \, & \triangleq \, \mathbb{E}_{p(\bar{\bm{w}})}  \left[  \log\left(\frac{p(\bar{\bm{w}})}{p^*}\right) \right] \\
    & \leq \, \frac{1}{2 \rho_U} \mathbb{E}_{p(\bar{\bm{w}})}  \left[ \left\|\nabla \log\left(\frac{p(\bar{\bm{w}})}{p^*}\right)\right\|_2^2 \right]. \label{eq:LSI}
\end{split}
\end{align}}

The average dynamics of the samples generated by the algorithm. The average of the samples generated by the active nodes within a cycle ($j=1$ to $T$) follows the following dynamics.
\begin{align}
    \bar{\bm{w}}_k^{j+1} &= \bar{\bm{w}}_k^j - \alpha_k \sum_{i= \mathcal{A}_k} \frac{1}{p_i} \widehat{\nabla E}_i(\bm{w}_{i,k}^j) + \sqrt{2\alpha_k} \bar{\bm{v}}_k^j, \label{eq:avg_dyn}
\end{align}
where $\bar{\bm{v}}_k^j \sim \mathcal{N}( \mathbf{0}_{d_w},I_{d_w})$. Define $\widetilde{\nabla E}_k^j = \sum_{i \HB{\in} \mathcal{A}_k} \frac{1}{p_i} \widehat{\nabla E}_i(\bm{w}_{i,k}^j) $ and rewrite~\eqref{eq:avg_dyn} as 
\begin{align}
    \bar{\bm{w}}_k^{j+1} &= \bar{\bm{w}}_k^j - \alpha_k \widetilde{\nabla E}_k^j + \sqrt{2\alpha_k} \bar{\bm{v}}_k^j, \label{eq:avg_dyn_2}
\end{align}
which in turn can be considered the discretization of the standard continuous-time Langevin dynamics given by
\begin{align}
    d\bm{w}_k(t) &= -\widetilde{\nabla E}_k dt + \sqrt{2}d\bm{B}(t), \label{eq:cont_lang}
\end{align}
where $\bm{B}(t)$ represents a Brownian motion. Note that $t$ is the continuous-time corresponding to $j$ in~\eqref{eq:avg_dyn_2}.

Denote by $\mathcal{B}_k$ the mini-batch used for computing all the stochastic gradients in the $k$-th cycle. We split $\widetilde{\nabla E}_k^j$ into multiple components as below for the purpose of analysis:
\begin{align} \label{eq:grad_terms}
\begin{split}
    \widetilde{\nabla E}_k^j &= \nabla E(\bar{\bm{w}}_k^j) - \xi(\bar{ \bm{w}}_k^j, \mathcal{A}_k) + \Phi(\bar{ \bm{w}}_k^j, \tilde{\mathbf{w}}_k^j, \mathcal{A}_k) \\
    &\qquad + \zeta(\bar{\bm{w}}_k^j,  \tilde{\mathbf{w}}_k^j, \mathcal{A}_k, \mathcal{B}_k),
\end{split}
\end{align}
where
\begin{align}
    \nabla E(\bar{\bm{w}}_k^j) &= \sum_{i=1}^{n} \nabla E_i (\bar{\bm{w}}_{k}^j), \label{eq:grad_terms_1} \\
    \xi(\bar{ \bm{w}}_k^j, \mathcal{A}_k) &= \nabla E(\bar{\bm{w}}_k^j) - \sum_{i\in \mathcal{A}_k} \frac{1}{p_i} \nabla E_i(\bar{\bm{w}}_k^j), \label{eq:grad_terms_2} \\
    \Phi(\bar{\bm{w}}_k^j, \tilde{\mathbf{w}}_k^j, \mathcal{A}_k) &= \sum_{i\in \mathcal{A}_k} \frac{1}{p_i} \left(\nabla E_i(\bm{w}_{i,k}^j) - \nabla E_i(\bar{\bm{w}}_k^j) \right), \label{eq:grad_terms_3} \\
    \zeta(\bar{\bm{w}}_k^j,  \tilde{\mathbf{w}}_k^j, \mathcal{A}_k, \mathcal{B}_k) &= \sum_{i\in \mathcal{A}_k} \frac{1}{p_i} \left(\widehat{\nabla E}_i( \bm{w}_{i,k}^j) - \nabla E_i( \bm{w}_{i,k}^j) \right). \label{eq:grad_terms_4}
\end{align}

Note that $\mathbb{E}_{\mathbf{A}}[\xi(\bar{ \bm{w}}_k^j, \mathcal{A}_k)] = 0$ and $\mathbb{E}_{\mathbf{B}}[\zeta(\bar{\bm{w}}_k^j,  \tilde{\mathbf{w}}_k^j, \mathcal{A}_k, \mathcal{B}_k)] = 0$ where $\mathbf{A}$ and $\mathbf{B}$ are the sets of all possible elements of $\mathcal{A}_k$ and $\mathcal{B}_k$, respectively. Each component of $\widetilde{\nabla E}_k^j$ encompasses a specific aspect of the algorithm: $\xi$ denotes the stochasticity due to gossip, $\Phi$ represents the gradient error due to distributed local computations, and $\zeta$ captures the stochastic mini-batch gradient effect.  

We next derive the propagation of the KL divergence within the cycle using the Fokker-Planck (FP) equation corresponding to~\eqref{eq:cont_lang}. Thereafter, the propagation of the KL divergence in each subsequent cycle can be developed as follows.
\begin{align}
   F\left(p(\bar{\bm{w}}_{k+1})\right) &\leq \exp(-\rho_U\alpha_kT) F\left(p(\bar{\bm{w}}_{k})\right) + \omega_k, \label{eq:KL_div_evo_1cycle}
\end{align}
where $\omega_k = 4\alpha_k^3\bar{L}^2 T (\bar{L}^2 C_{_{\bar{\bm{w}}}} + C_\xi + C_\zeta) + 2\alpha_k^2\bar{L}^2 Td_w + \frac{n (n-1)T}{p_m^2} (4\alpha_k^3L^2 \bar{L}^2 + \alpha_k L^2) \mathbb{E}[ \|\tilde{\mathbf{w}}_k\|^2]$. From~\eqref{eq:KL_div_evo_1cycle}, the rate of convergence in terms of the KL divergence is obtained below.

\begin{theorem} \label{thm:Convergence}
Suppose that Assumptions~\ref{assump:Lipz}--\ref{assump:LSI} hold and Conditions~\ref{cond:1}-\ref{cond:3} are satisfied, then $F\left( p(\bar{\bm{w}}_{k+1})\right)$ satisfies
\begin{align}
    F\left( p(\bar{\bm{w}}_{k+1})\right) &\leq  \frac{C_{_{F_1}}}{\exp \left( \frac{a\rho_UT (k+1)^{1-\delta_\alpha} }{1-\delta_\alpha} \right)} + \frac{C_{_{F_2}}}{(k+1)^{\delta_\alpha}}, \label{eq:convergence_rate}
\end{align}
where 
\begin{align}
%\begin{split}
    &C_{_{F_1}} = \exp\left(\frac{a\rho_UT}{1-\delta_\alpha}\right) \Bigg[ F\left( p(\bar{\bm{w}}_0)\right) + \sum_{\ell=0}^{\bar{k}} \frac{C_\omega}{\ell^{2\delta_\alpha}} \exp \Bigg(a\rho_UT \times \nonumber \\
    &\, \left(1+\frac{\ell^{1-\delta_\alpha}}{1-\delta_\alpha} \right) \Bigg)\Bigg] + \bar{W}_1\exp \left( \frac{a\rho_UT(2-\delta_\alpha)}{1-\delta_\alpha} \right) \times \nonumber \\
    &\, \Bigg[ \sum_{\ell=0}^{\bar{\ell}} \exp\left(\frac{a\rho_UT}{1-\delta_\alpha} \ell^{1-\delta_\alpha} - |\ln \bar{\lambda}| \ell \right) - \Bigg( \frac{a\rho_UT}{1-\delta_\alpha} \ell^{-\delta_\alpha} \nonumber \\
    &\, - |\ln \bar{\lambda}| \Bigg)^{-1} \Bigg], \\ 
%\end{split} \\
    C_{_{F_2}} &= \frac{2C_\omega \delta_\alpha}{a\delta_\alpha \rho_UT} \exp \left( \frac{a\rho_UT}{1-\delta_\alpha} 2^{1-\delta_\alpha}\right),
\end{align}
in which $C_\omega = 4a^3\bar{L}^2 T (\bar{L}^2 C_{_{\bar{\bm{w}}}} + C_\xi + C_\zeta) + 2a^2\bar{L}^ T 2d_w + \frac{n (n-1)T}{p_m^2} (4a^3L^2 \bar{L}^2 + a L^2) W_2$ and $\bar{W}_1 = \frac{n(n-1)T}{p_m^2}(4a^3L^2 \bar{L}^2 + a L^2) W_1$. 
\end{theorem}

Equation~\eqref{eq:convergence_rate} in Theorem~\ref{thm:Convergence} proves that the KL divergence of the distribution of the average of the samples from the true posterior $p^*$ converges to zero asymptotically. We also observe that the convergence of the consensus error and the KL divergence occurs at the polynomial rate of $\mathcal{O}\left( k^{-\delta_\alpha} \right)$. Hence, a decreasing step size for the algorithm is essential for asymptotic convergence since the polynomial rate is dependent only on the annealing rate of the step size. The first term in~\eqref{eq:convergence_rate} has an exponential rate which is characteristic of the centralized ULA, however, the consensus error decay rate slows down the overall convergence of the algorithm to the same polynomial rate. 

\subsection{Discussion on Results}
{We note from the proof of Theorem~\ref{thm:Convergence} that asymptotic convergence would hold even if any other source of stochasticity is introduced in the gradient of the average dynamics as long as it is zero-mean and has a bounded {second moment} with a compromise of a higher uncertainty in the result for the same number of cycles. This was also speculated in~\cite{welling2011bayesian}.

A caveat to the proposed algorithm is that in any particular cycle both gossiping agents must synchronize the gradient step size %(though it decreases with cycles) 
and the number of computations. %(though this can vary between cycles). 
Our mathematical analysis quantifies the constraint on the number of local computations as well. We attempted to address the convergence by allowing different number of local computations and different gradient step sizes for the two gossiping agents in each cycle. However, we faced the following issues. The KL divergence of the average of the samples asymptotically diverges if the gradient step size is not synchronized. Though this can be achieved for conventional optimization problems~\cite{ram2009asynchronous}, since the step size controls the variance of the injected noise as well, this leads to instability in our analysis. Also, if the gossiping agents are allowed different number of computations in the same cycle, it appears to introduce a bias in the convergence error, but predominantly the major issue is to establish a bound for the expectation of the second moment of the average of the generated samples, which is needed for consensus.

Despite the aforementioned constraints, the gossip protocol allows for asynchronous implementation and the execution of multiple cycles can overlap. Thus, the agents do not need to know when a cycle starts and terminates, which provides a higher degree of asynchrony and expedites the process. {If an active agent chooses to gossip with a neighbor which happens to be already engaged in another ongoing cycle, it simply needs to choose a different neighbor. Note that this re-selection slightly changes the pre-determined expression of $p_i$ mentioned in Section~\ref{sec:Goss_protocol}. However, the probability of an already engaged agent being chosen for gossip is given as $\frac{1}{n(n-2)} \left(1+  \sum_{j\in \mathcal{N}_i} \frac{1}{|\mathcal{N}_j|}\right) \left(  \sum_{j\in \mathcal{N}_i/\{\mathcal{A}_k\}} \frac{1}{|\mathcal{N}_j|} \right)$. Thus, for a large communication graph with numerous agents and connections, this possibility is low.}

%% file: Experiments.tex
\section{Experiments} \label{sec:Experiments}
We present our simulation experiments and their results and associated conclusions in this section. Section~\ref{sec:Ex0} presents a simple academic problem to highlight the benefits of the proposed algorithm. In Section~\ref{sec:Ex1} we apply the algorithm to real world classification problems. The results in this section verifies the efficacy on classification problems with real data. %{In all the figures, the $x$-axis of all plots denote communication \emph{cycles}.} \HB{If this is the case, you don't need any $x$ labels.}

\subsection{1D Gaussian toy Problem} \label{sec:Ex0}
We apply the algorithm to a simple 1D Gaussian toy problem presented in~\cite{teh2016consistency}. Since we can analytically compute the posterior distribution of this problem, we  perform a direct comparison of the posterior approximated by our algorithm with the true analytical posterior. Let 
\begin{align} 
    \theta &\sim \mathcal{N}(0,\sigma_{\theta}^2) \label{eq:G_toy1} \\
    x_i|\theta &\sim \mathcal{N}(\theta, \sigma_x^2), \qquad i=1,2,\ldots, N. \label{eq:G_toy2}
\end{align}
where we use $\sigma_{\theta} = 1$, $\HB{\sigma}_x = 5$ and $N=50$. %\HB{only 5 data points?} 
The analytical expression for the posterior is given as 
\begin{align}
    \pi &= \mathcal{N} \left(\mu_p,\sigma_p^2 \right) = \mathcal{N} \left( \frac{\sum_{i=1}^Nx_i}{\frac{\sigma_x^2}{\sigma_\theta^2}+N}, \left(\frac{1} {\sigma_\theta^2} + \frac{N}{\sigma_x^2} \right)^{-1} \right). \label{eq:G_ana_pos}
\end{align}
The true value of $\theta=0.0851$, while $\mu_p=1.5754$ and $\sigma_p^2 = 0.5774$. Since the true posterior in~\eqref{eq:G_ana_pos} is a Gaussian distribution, we  compare our simulation results by the KL divergence of the distribution generated by the MCMC samples from the analytical posterior in~\eqref{eq:G_ana_pos}. Denote the mean and the variance of the simulation samples by $\mu_s$ and $\sigma_s^2$ respectively. Then the KL divergence is given by $\log \frac{\sigma_p}{\sigma_s} + \frac{\sigma_s^2 + (\mu_s-\mu_p)^2}{2\sigma_p^2} - \frac{1}{2}$. 

The entire data was distributed equally among the $5$ agents, each receiving  $10$ data points. The communication topology is a ring graph where each agent communicates with two neighbors. The hyperparameters used are: $a=10^{-4}$, $b = 0.5$, and $\delta_{\alpha}=0.01$. Each agent uses $f=0.1$ of its data (i.e., mini-batch of $1$ data point) for each local computation and perform $5000$ Monte Carlo chains for the simulations. The simulation is run for  $10^4$ cycles.

The results of this experiment averaged over all MC chains for all the agents are provided in Fig.~\ref{fig:GT1} for $T= \{1,3,5\}$. The decreasing trend of KL divergence observed in each case suggests successful learning of the posterior distribution by the algorithm. Furthermore, a clear trend of faster learning per cycle can be observed as $T$ increases from $1$ to $5$,  illustrating the benefits of performing multiple local multiple computations per cycle. 

\begin{figure}[htpb]
\begin{center}
    %\centering
    \includegraphics[width=1.0\linewidth]{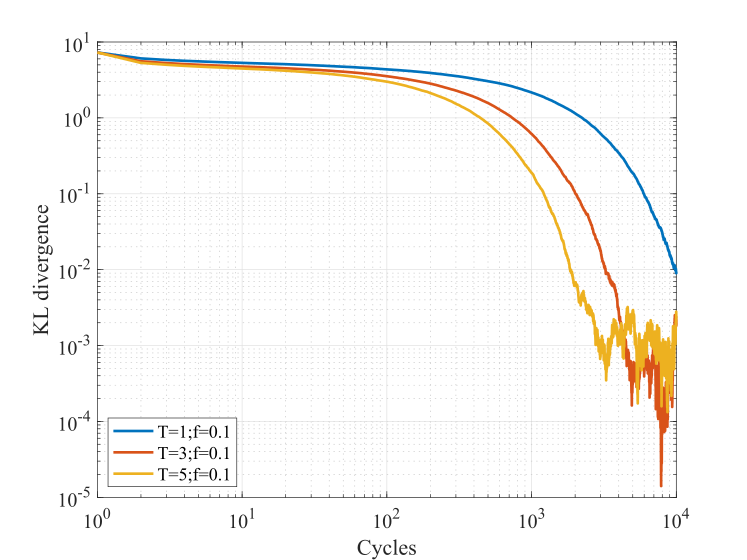}
    % \put(-232,128){{\textrm{(a)}}}
    % \put(-125,128){{\textrm{(b)}}}
    \caption{Comparison of the KL divergence using gossip ULA for different number of local computations per cycle; averaged over all agents and all chains.}
    \label{fig:GT1}
\end{center}  
\end{figure}

\subsection{Classification} \label{sec:Ex1_2}
In this section, we demonstrate the proposed algorithm on two real world data sets: (i) Gamma Telescope data set and (ii) mHealth data set from the UCI ML repository for performing classification problem. For both cases, data is distributed among $6$ agents and a Gaussian prior $\bm{w} \sim \mathcal{N}(\mathbf{0}_{d_w}, 20^2 I_{d_w})$ is used. For communication, we assume a ring graph topology for all the experiments. {In all the figures in this section, the horizontal axis denotes the number of \emph{cycles} while the vertical axis denotes the \emph{accuracy} in percentage.} 

\subsubsection{\textbf{Binary Classification}} \label{sec:Ex1}
We use the Gamma Telescope data set\footnote{https://archive.ics.uci.edu/ml/datasets/magic+gamma+telescope} in our numerical experiments. This represents the data from the imaging technique used on reading on the Cherenkov gamma telescope. {Each data point originally consists of $10$ attributes and an extra attribute `1' is augmented to each data point, thus, $d_w = 11$.} The problem involves detecting the readings that correspond to gamma signals {(labeled `1')} from among the background hadron signals {(labeled `0')}. Thus, it is a binary classification problem and we use a binary logistic regression model to perform the classification. {Define a mini-batch as $\hat{\bm{X}} = \{x^j\}_{j=1}^{m}$ where $x^j=\{x_I^j, x_L^j\}$, and $x_I^j \in \mathbb{R}^{11}$ and $x_L^j \in \{0,1\}$ are the inputs and corresponding labels of the data, respectively. Then, for any $\bm{w} \in \mathbb{R}^{11}$ the logistic function is given as $\ell(\bm{w}, x_I^j) = \frac{1}{1 + \exp(-\bm{w}^\top x_I^j)}$ and the likelihood as $p(\hat{\bm{X}}|\bm{w}) = - \sum_{j=1}^{m} x_L^j \ln \ell(\bm{w}, x_I^j) + (1-x_L^j) \ln (1- \ell(\bm{w}, x_I^j))$. Each element of $\bm{w}$ was initialized with a standard normal distribution.}

The data in this case is not distributed, so we randomly split the data into $6$ parts that correspond to each agent's local data set. The values of hyper-parameters used are $a=10^{-5}$, $b = 0.5$ and $\delta_{\alpha}=0.5$. The simulations are run for $10$ trials and the average accuracy is reported. Each trial consisted of $150$ cycles. 

\begin{figure}[htpb]
\begin{center}
    \includegraphics[width=1.0\linewidth]{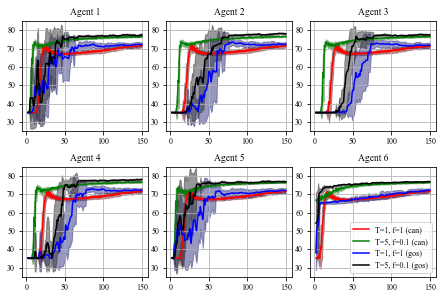}
    \caption{Comparison of the accuracy on the test data set averaged over $10$ trials for each agent. (Here, `can' and `gos' denote canonical and gossip respectively.)}
    \label{fig:LR_tel_comp}
\end{center}  
\end{figure}

Fig.~\ref{fig:LR_tel_comp} shows the results of the accuracy performance %on the Gamma Telescope data set 
under different conditions. Specifically, we use the synchronized case with $T=1$ as the benchmark for comparison. We observe that the synchronized case with $T=5$ has faster convergence in the initial phase as expected from our analysis. In contrast, the gossip cases have slower convergence with higher variance in the initial phase since only $2$ agents make updates at every iteration. However, the final accuracy matches the corresponding synchronized cases. Between the two gossip cases, the $T=5$ case is faster than $T=1$ in the initial phase. Also, for the $T=5$ cases, we use stochastic gradient with mini-batches $1/10$-th the size of the entire data for each agent to expedite the multiple local computations. Fig.~\ref{fig:LR_tel_comp} clearly shows the efficacy of gossip, especially when performed with multiple local computations and stochastic gradient. 

\subsubsection{\textbf{mHealth Data set}} \label{sec:Ex2}
We demonstrate our algorithms using the mHealth data set\footnote{http://archive.ics.uci.edu/ml/datasets/mhealth+dataset}. The data consists of readings from $23$ different sensors strapped on human subjects and the task is to detect the activity being performed by the subject. There is a known list of $12$ recorded activities. {With an augmented attribute `1', each data point has $24$ attributes and number of classes is $12$ (labeled from `0' to `11'). We use cross-entropy likelihood function for this multi-class classification. Let $\hat{\bm{X}} = \{x_I^j, x_L^j\}_{j=1}^m$ where $x_I^j \in \mathbb{R}^{24}$ and $x_L^j \in \mathbb{R}^{12}$ are bit vectors for each label from `0' to `11', and $\bm{w} \in \mathbb{R}^{12*24}$ (i.e., $d_w = 288$) which can be reshaped into a matrix $\bm{W} \in \mathbb{R}^{12 \times 24}$. The softmax probabilities for each class corresponding to $x_I^j$ are computed as $\ell(\bm{W},x_I^j) = \exp(\bm{W}x_I^j)/\| \exp (\bm{W}x_I^j)\|_1 \in \mathbb{R}^{12}$, where $\exp(\cdot)$ denotes the element-wise exponential while $\|\cdot\|_1$ denotes the L1 norm. The cross-entropy likelihood function can then be expressed as $p(\hat{\bm{X}} |\bm{W}) = -\sum_{j=1}^m (x_L^j)^\top \ell(W,x_I^j)$. Each element of $\bm{w}$ was initialized with a Laplace distribution with mean equals $0$ and scale parameter equals $1$.}

The data set contains some unlabelled data. For our experiments, we remove all unlabelled data (from both training and testing). The data consists readings from $10$ different subjects and hence is inherently distributed. We take the data from the first $6$ subjects as the local training data for the $6$ agents, while we use the data from Subject $10$ as the test data. {The hyper-parameters used are $a=10^{-5}$, $b=0.8$ and $\delta_\alpha = 0.5$. A total of $10$ trials were performed in each case.}

We perform the simulations for a few different cases of local computations. First, we do uniform computations where the active agents perform a pre-determined ($T \in \{1,3,5\}$) number of computations, which remains the same at each cycle. Next, we perform variable numbers of local computations where each active agent performs  (same) randomly determined number of local computations between $1$ and $10$, which varies
at each cycle. The results of the accuracy plots on the test data in Fig.~\ref{fig:mHD_1}-\ref{fig:mHD_2}. 

Fig.~\ref{fig:mHD_1}-\ref{fig:mHD_3} compare the accuracy results on the test data for pre-determined and varying number of local computations where each agent receives only $100$ data points, and a mini-batch size of $10$ is utilized to compute the gradient. Fig.~\ref{fig:mHD_1} clearly demonstrates the faster training gain for a higher $T$ value. Although the difference between $T=3$ and $T=5$ is less pronounced, each case performs significantly better than the $T=1$ case, which indicates that even a few rounds of multiple computations can make a beneficial impact. Furthermore, Fig.~\ref{fig:mHD_2} shows that similar advantages are observed in the case where the number of local computations varies in each cycle (dynamic mode), which proves that a pre-determined value of $T$ is not mandatory. {For our simulations, $T$ was randomly chosen between $1$ and $10$.} Additionally, in Fig.~\ref{fig:mHD_3}, we present the results of the same experiment performed with a much larger training data set per agent, in this case $2000$ training data points per agent. The test data remains the same as earlier. The agents still use stochastic gradient with a mini-batch size $10$. {Overall similar trends are observed. Although the expedition in the rate of convergence by multiple local computations is less pronounced as the larger data set facilitates better learning right from the start, the faster convergence is still noticeable, especially for agent $3$. Additionally, the variance among the trials appears to be lower in the case of multiple computations.}  

\begin{figure}[htpb]
\begin{center}
    \includegraphics[width=\linewidth]{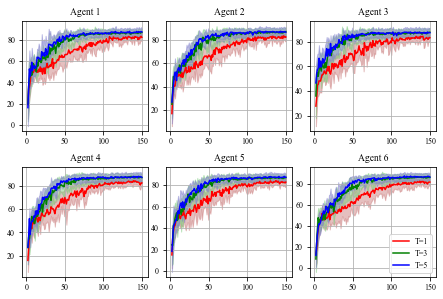}
     \caption{Comparison of the accuracy on the test data set for $6$ agents for $T=1$, $T=3$ and $T=5$ local computations via gossip ULA.}
    \label{fig:mHD_1}
\end{center}  
\end{figure}

\begin{figure}[htpb]
\begin{center}
    \includegraphics[width=\linewidth]{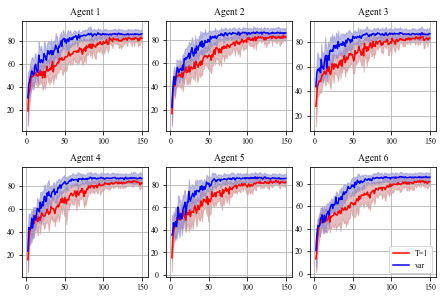}
     \caption{Comparison of the accuracy on the test data set for $6$ agents for $T=1$ and varying $T$ (denoted by `var') local computations via gossip ULA.}
    \label{fig:mHD_2}
\end{center}  
\end{figure}

\begin{figure}[htpb]
\begin{center}
    \includegraphics[width=\linewidth]{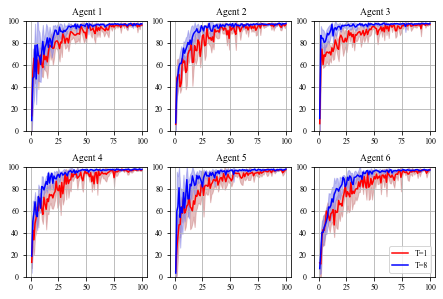}
     \caption{Comparison of the accuracy on the test data set for $6$ agents for $T=1$ and $T=8$ local computations via gossip ULA for a larger training data set.}
    \label{fig:mHD_3}
\end{center}  
\end{figure}

\subsection{Discussion on experiments}
The most important aspect of the experiments is the execution of a gossip-based asynchronous ULA which drastically reduces communication overhead. Since all the agents are only occasionally making updates, leveraging multiple local computations (which are often computationally cheaper and faster than communication in distributed setups) per cycle increases speed of convergence. Since these local computations are gradient based, we are able to expedite the process by incorporating a mini-batch stochastic gradient which is reused for all the local computations within a cycle. For our simulations, the usage of mini-batches of just $1/10$-th the size of the total data is sufficient to give useful results. 

We also observe that difference in the rate of convergence is more pronounced with smaller training data sets, making it an ideal framework for quick learning in limited data paradigm. However, we speculate that there can be more complex models with more learning parameters where multiple computations per cycle may be advantageous even with a reasonably large training data set. {Even in the case of large training data set, there may be instances where some agents still benefit from multiple computations, as evidenced in the case of agent $3$ in Fig.~\ref{fig:mHD_3}.} Thus, in scenarios where the communication overhead heavily dominates the computational overhead, performing multiple asynchronous computations is an effective strategy to speed up convergence.

%% file: Conclusions.tex
\section{Conclusions} \label{sec:Conclusions}

{In this paper, we propose an asynchronous distributed Bayesian learning framework using ULA. %The major contribution of this work is the development of the asynchronous distributed implementation using gossip protocol which models asynchrony as a Poisson's process. 
The proposed algorithm reduces communication overhead at each cycle, and avoids time consumption due to idling during communication delays, which can result in reduction in training time. The algorithm also employs multiple `\emph{fast}' local computations in each cycle. {Our analysis and experiments show that at the initial phase of training, the proposed algorithm exhibits a convergence rate proportional to the exponential of the number of local computations, which is faster than the canonical distributed ULA.} This expedition can be practically significant, especially under limited data constraints. Furthermore, the local computations are made computationally cheaper by using mini-batch stochastic gradients and reusing the same mini-batch throughout the entire cycle. Our theoretical analysis shows that despite the additional stochasticities due to gossip and mini-batch gradients, both asymptotic consensus and convergence rates remain the same as the canonical distributed ULA up to certain constants. Our analysis approach can be easily extended to algorithms with other sources of zero-mean stochasticities that have bounded second moments.}

%% file: Appendix.tex
%\section{Appendix} 
\section{Supplementary Material}
\label{sec:Appendix}
In this section, we present proof of our theorems and some supplementary materials used in the proof.
%\HB{It's strange that we don't have proof for the theorems? How are these lemmas used to establish the results? We would need complete proofs for both theorems. }

\subsection{Proof of Theorem~\ref{thm:consensus}}
The $j$-th local computation (where $0\leq j<T$) within the $k$-th cycle for all agents in a vectorized form as below:
\begin{align}
    &\mathbf{w}_{k}^{j+1} = \mathbf{w}_k^j - n\alpha_k S_k \hat{\mathbf{g}}_{k}^j + \sqrt{\alpha_kn^2} S'_k \mathbf{v}_{k}^j, \label{eq:con_syn} 
\end{align}
where $\hat{\mathbf{g}}_{k}^j = \Big[\widehat{\nabla E}_1(\bm{w}_1^j)^\top, \ldots, \widehat{\nabla E}_n(\bm{w}_n^j)^\top \Big]^\top \in \mathbb{R}^{nd_w}$, $\mathbf{v}_{k}^j = {[{\bm{v}_1^j}^\top, \ldots, {\bm{v}_n^j}^\top ]}^\top \in \mathbb{R}^{nd_w}$, $S_k = \left( \frac{1}{p_{i_k}} \mathbf{e}_{i_k}\mathbf{e}_{i_k}^\top + \frac{1}{p_{j_k}} \mathbf{e}_{j_k} \mathbf{e}_{j_k}^\top \right) \otimes \mathbf{1}_{d_w} \in \mathbb{R}^{nd_w \times nd_w}$ and $S'_k = \left( \mathbf{e}_{i_k} \mathbf{e}_{i_k}^\top + \mathbf{e}_{j_k} \mathbf{e}_{j_k}^\top \right) \otimes \mathbf{1}_{d_w} \in \mathbb{R}^{nd_w \times nd_w}$. Next, using~\eqref{eq:con_syn} iteratively for $0\leq j<T$ within the $k$-th cycle yields
\begin{align}
    \mathbf{w}_{k+1} &= \mathcal{W}_k \mathbf{w}_k - \alpha_kn S_k\mathbf{G}_k + \sqrt{\alpha_kn^2} S'_k\mathbf{V}_k, \label{eq:con_cycle}
\end{align}
where $\mathcal{W}_k = (I_n - \beta \mathcal{L}_k) \otimes I_{d_w} \in \mathbb{R}^{nd_w \times nd_w}$, $\mathbf{G}_k = \sum_{j=0}^{T-1} \hat{\mathbf{g}}_{k}^j \in \mathbb{R}^{nd_w}$ and $\mathbf{V}_k = \sum_{j=0}^{T-1} \mathbf{v}_{k}^j \in \mathbb{R}^{nd_w}$. 

Define $M = I_n - \frac{1}{n} \mathbf{1}_n \mathbf{1}_n^\top$. Pre-multiplying~\eqref{eq:con_cycle} with $(M \otimes I_{d_w})$ gives 
\begin{align} \label{eq:con_err_cycle} 
\begin{split}
    \tilde{\mathbf{w}}_{k} &= \mathcal{W}_k \tilde{\mathbf{w}}_k - (M \otimes I_{d_w}) \Big(\alpha_knS_k \mathbf{G}_k + \sqrt{\alpha_kn^2} S'_k\mathbf{V}_k \Big), 
\end{split}
\end{align}
where we use $(M \otimes I_{d_w})\mathcal{W}_k = \mathcal{W}_k (M \otimes I_{d_w})$. Next, taking norm of~\eqref{eq:con_err_cycle} yields
\begin{align}
    \|\tilde{\mathbf{w}}_{k+1}\| &\leq \|\mathcal{W}_k \tilde{\mathbf{w}}_k\| + \alpha_kn \|S_k \mathbf{G}_k\| + \sqrt{\alpha_kn^2} \|S'_k\mathbf{V}_k\|, \label{eq:con_err_syn_norm}
\end{align}
where we use the relation $\|(M \otimes I_{d_w})) \mathbf{x}\| \leq \|(M \otimes I_{d_w})\| \|\mathbf{x}\|$ (for any $\mathbf{x} \in \mathbb{R}^{nd_w}$) and the result $\|M \otimes I_{d_w}\|=1$. Thereafter, using the relation $(x+y)^2 \leq (\theta+1)x^2 + \left( \frac{ \theta+1}{\theta}\right)y^2$ (for any $\theta>0$) on~\eqref{eq:con_err_syn_norm} twice gives
\begin{align} \label{eq:con_err_syn_norm_sq}
    \|\tilde{\mathbf{w}}_{k+1}\|^2 &\leq (\theta+1)^2 \|\mathcal{W}_k \tilde{\mathbf{w}}_k\|^2 +  \frac{\alpha_k^2n^2(\theta+1)^2}{\theta} \|S_k\mathbf{G}_k\|^2 \nonumber \\
    &\qquad + \alpha_k n^2 \left(\frac{\theta+1}{\theta} \right) \|S'_k\mathbf{V}_k\|^2. 
\end{align}

The next objective is to obtain a bound for $\| S_k \mathbf{G}_k\|$. Define $\mathbf{g}_{k}^j = \Big[{\nabla E}_1(\bm{w}_{1,k}^j)^\top, \ldots, {\nabla E}_n(\bm{w}_{n,k}^j)^\top \Big]^\top \in \mathbb{R}^{nd_w}$, $\bar{\mathbf{g}}_{k}^j = \Big[{\nabla E}_1(\bar{\bm{w}}_k^j)^\top, \ldots, {\nabla E}_n(\bar{\bm{w}}_k^j)^\top \Big]^\top \in \mathbb{R}^{nd_w}$ and $\mathbf{g}^* = \Big[{\nabla E}_1(\bm{w}_1^*)^\top, \ldots, {\nabla E}_n(\bm{w}_n^*)^\top \Big]^\top = \mathbf{0}_{nd_w} \in \mathbb{R}^{nd_w}$. {Let $\|\mathbf{G}'_k\| = \sum_{j=0}^{T-1} \|S_k \hat{\mathbf{g}}_k^j\|$ which gives $\|S_k\mathbf{G}_k\| \leq \|\mathbf{G}'_k\|$.
Thus, we have 
\begin{align}
    \|\mathbf{G}'_k\| &\leq \sum_{j=0}^{T-1} \left\|S_k \left( (\hat{\mathbf{g}}_k^j - {\mathbf{g}}_k^j) + (\mathbf{g}_k^j - \bar{\mathbf{g}}_k^j) + (\bar{\mathbf{g}}_k^j - \mathbf{g}^*) \right) \right\|, \nonumber \\
    &\leq \sum_{j=0}^{T-1} \|S_k (\hat{\mathbf{g}}_k^j - {\mathbf{g}}_k^j) \| + \sum_{j=0}^{T-1} \|S_k({\mathbf{g}}_k^j - \bar{\mathbf{g}}_k^j) \| \nonumber \\
    &\qquad \qquad + \sum_{j=0}^{T-1} \|S_k(\mathbf{g}_k^j - \bar{\mathbf{g}}_k^j) \|, \nonumber \\
    &\leq \sum_{j=0}^{T-1} \|\boldsymbol{\zeta}_{\mathcal{A}_k}^j\| + \frac{L}{p_m} \Bigg(\sum_{j=0}^{T-1} \|\tilde{\mathbf{w}}_k^j\| + \sum_{j=0}^{T-1} \|\mathbf{1}_{\mathcal{A}_k} \otimes \bar{\bm{w}}_k^j\| \nonumber \\
    &\qquad \qquad + T \|\mathbf{w}^*\| \Bigg), \label{eq:G_bound_1}
\end{align}
where $\boldsymbol{\zeta}_{\mathcal{A}_k}^j \triangleq S_k (\hat{\mathbf{g}}_k^j - {\mathbf{g}}_k^j)$ and $\mathbf{1}_{\mathcal{A}_k} \in \mathbb{R}^n$ is a vector of all zeros except $1$s at positions $\mathcal{A}_k$. In~\eqref{eq:G_bound_1}, we use Assumption~\ref{assump:Lipz}. %and the fact that $\|\tilde{\mathbf{w}}_k^j\| \leq \|\tilde{\mathbf{w}}_k^T\| \equiv \| \tilde{\mathbf{w}}_{k+1}\|$ for $0\leq j\leq T$ for any $k\geq0$ since the consensus error will keep increasing within the cycle \HB{Did we establish this?}. 
Now, pre-multiplying $(M\otimes I_{d_w})$ to~\eqref{eq:con_syn} and iteratively using it and thereafter taking norm, yield the following bound for $\|\tilde{\mathbf{w}}_k^j\|$:
\begin{align} \label{eq:bar_w}
\begin{split}
    \|\tilde{\mathbf{w}}_k^j\| &\leq \|\mathcal{W}_k \tilde{\mathbf{w}}_k\| + n \alpha_k \sum_{\ell=0}^{j-1} \|S_k \hat{\mathbf{g}}_k^\ell\| + \sqrt{\alpha_k n^2} \sum_{\ell=0}^{j-1} \|S'_k \mathbf{v}_k^\HB{\ell}\|.
\end{split}
\end{align}
Further, let $\|\mathbf{V}'_k\| = \sum_{j=0}^{T-1} \|S'_k {\mathbf{v}}_k^j\|$ which gives $\|S'_k \mathbf{V}_k\| \leq \|\mathbf{V}'_k\|$. Then, from~\eqref{eq:bar_w} we have
\begin{align} \label{eq:bar_w_sum}
\begin{split}
    \sum_{j=0}^{T-1} \| \tilde{\mathbf{w}}_k^j\| &\leq T\|\mathcal{W}_k \tilde{\mathbf{w}}_k\| + n \alpha_k T \|{\mathbf{G}}'_k\| + \sqrt{\alpha_k n^2}T \|\mathbf{V}'_k\|,
\end{split}
\end{align}
since for any $0\leq j <T$, we have $\sum_{\ell=0}^{j}\|S_k \hat{\mathbf{g}}_k^\ell\| \leq \sum_{\ell=0}^{T-1}\|S_k \hat{\mathbf{g}}_k^\ell\| = \|\mathbf{G}'_k\|$ and similarly $\sum_{\ell=0}^{j}\|S'_k {\mathbf{v}}_k^\ell\| \leq \|\mathbf{V}'_k\|$. Now, substituting~\eqref{eq:bar_w_sum} in~\eqref{eq:G_bound_1} yields the desired bound:
\begin{align} \label{eq:G_bound_3}
\begin{split}
    \|\mathbf{G}'_k\| &\leq \left(\frac{1}{1-n\alpha_k LT/p_m} \right) \Bigg(\sum_{j=0}^{T-1} \|\boldsymbol{\zeta}_{\mathcal{A}_k}^j\| \\
    &\qquad + \frac{LT}{p_m} \|\mathcal{W}_k \tilde{\mathbf{w}}_k\| + \frac{LT}{p_m} \sqrt{\alpha_kn^2} \|\mathbf{V}'_k\| \\
    &\qquad + \frac{L}{p_m} \sum_{j=0}^{T-1} \|\mathbf{1}_{\mathcal{A}_k} \otimes \bar{\bm{w}}_k^j\|  + \frac{LT}{p_m} \|\mathbf{w}^*\| \Bigg), 
\end{split}
\end{align}
which further gives
\begin{align} \label{eq:G_bound_4}
\begin{split}
    &\|\mathbf{G}'_k\|^2 \leq \left(\frac{1}{1-n\alpha_k LT/p_m} \right)^2 \Bigg[ 5\Bigg(\sum_{j=0}^{T-1} \|\boldsymbol{\zeta}_{\mathcal{A}_k}^j\| \Bigg)^2 \\
    &\qquad + \frac{5L^2T^2}{p_m^2} \|\mathcal{W}_k \tilde{\mathbf{w}}_k\|^2 + \frac{5 \alpha_k n^2L^2T^2}{p_m^2} \|\mathbf{V}'_k\|^2 \\
    &\qquad + \frac{5L^2}{p_m^2} \Bigg(  \sum_{j=0}^{T-1} \|\mathbf{1}_{\mathcal{A}_k} \otimes \bar{\bm{w}}_k^j\| \Bigg)^2  + \frac{5 L^2T^2}{p_m^2} \|\mathbf{w}^*\|^2 \Bigg].
\end{split}
\end{align}
Noting that $\|S_k \mathbf{G}_k\|^2 \leq \|\mathbf{G}'_k\|^2$ and $\|S'_k \mathbf{V}_k\|^2 \leq \|\mathbf{V}'_k\|^2$; and substituting~\eqref{eq:G_bound_4} in~\eqref{eq:con_err_syn_norm_sq} result in
\begin{align} \label{eq:cons_error_sq_1}
%\begin{split}
    &\|\tilde{\mathbf{w}}_{k+1}\|^2 \leq \Bigg[ 1 + \frac{5\alpha_k^2n^2L^2T^2}{\theta p_m^2(1-\alpha_knLT/p_m)^2}\Bigg] (\theta+1)^2 \|\mathcal{W}_k \tilde{\mathbf{w}}_k\|^2 \nonumber \\
    &\quad + \frac{5\alpha_k^2n^2T(\theta + 1)^2}{\theta(1-\alpha_knLT/p_m)^2} \Bigg[\sum_{j=0}^{T-1} \| \boldsymbol{\zeta}_{\mathcal{A}_k}^j \|^2 + \frac{L^2T}{p_m^2}\| \mathbf{w}^* \|^2 \nonumber \\
    &\quad + \frac{nL^2}{p_m^2} \sum_{j=0}^{T-1} \|\bar{\bm{w}}_k^j \|^2 \Bigg] + \left(\frac{\theta+1}{\theta}\right) \Bigg[ \frac{5\alpha_k^3 n^4L^2T^2 (\theta+1)}{p_m^2 (1-\alpha_k nlT/p_m)^2} \nonumber \\
    &\quad + \alpha_kn^2 \Bigg] \|S'_k \mathbf{V}'_k\|^2.
%\end{split}
\end{align}}
Denote by $\mathcal{F}_k$ the history of the sequence of samples and active agents $\{\mathbf{w}_0, \ldots, \mathbf{w}_k, \mathcal{A}_0, \ldots, \mathcal{A}_k\}$. The expectation $\mathbb{E}[\cdot|\mathcal{F}_k]$ of the individual terms in~\eqref{eq:cons_error_sq_1} are analyzed separately. The following results can be shown:
\begin{align} \label{eq:cons_error_sq_1.1}
    \mathbb{E}[\|\mathcal{W}_k \tilde{ \mathbf{w}}_k\|^2|\mathcal{F}_k] &= \tilde{\mathbf{w}}_k^\top \mathbb{E}[\|\mathcal{W}_k^\top \mathcal{W}_k\|^2 |\mathcal{F}_k] \tilde{\mathbf{w}}_k \leq  \lambda \| \tilde{\mathbf{w}}_k\|^2.
\end{align}
Also, using Assumption~\ref{assump:sto_grad}, the following inequality holds.
\begin{align} \label{eq:cons_error_sq_1.2}
    \mathbb{E} \left[ \sum_{j=0}^{T-1} \left\|\boldsymbol{\zeta}_{_{ \mathcal{A}_k}}^j\right\|^2 \Bigg| \mathcal{F}_k \right] &= \mathbb{E} \left[\sum_{j=0}^{T-1} \sum_{i\in \mathcal{A}_k} \left\| \boldsymbol{\zeta}_{i, k}^j \right\|^2 \Bigg|\mathcal{F}_k \right] \leq 2T C_\zeta.
\end{align}
Note that the remaining terms $\|\mathbf{w}^*\|^2$, $\sum_{j=0}^{T-1} \|\bar{\bm{w}}_k^j\|^2$ and $\|\mathbf{V}'_k\|^2$ are independent of $\mathcal{F}_k$. Since $\|\mathbf{w}^*\|^2$ is a constant, there exists some $C_{_{\mathbf{w}^*}}>0$ such that $\mathbb{E} \left[\left\|\mathbf{w}^*\right\|^2\right] \leq C_{_{\mathbf{w}^*}}$. It also follows that $\mathbb{E} \left[\|\mathbf{V}'_k\|^2\right] \leq 2T^2d_w$. Finally, we have $\mathbb{E}[\|\bar{\bm{w}}_k^j\|^2] \leq C_{_{\bar{\bm{w}}}}$ for any $0<k,\, 0\leq j<T$ from Lemma~\ref{lemma:bounded_avg_samples}. Thus, $\mathbb{E} \left[\sum_{j=0}^{T-1} \|\bar{\bm{w}}_k^j\|^2 \right] \leq TC_{_{\bar{\bm{w}}}}$. Substituting the above results in~\eqref{eq:cons_error_sq_1} and taking the total expectation results in~\eqref{eq:con_err_syn_norm_sq_f}. Thereby, applying Lemma~\ref{lemma:consensus_rate_lemma} to~\eqref{eq:con_err_syn_norm_sq_f} yields the result~\eqref{eq:consensus_rate} in Theorem~\ref{thm:consensus}.

\subsection{Proof of Theorem~\ref{thm:Convergence}}
We start by establishing the Fokker-Planck (FP) equation corresponding to~\eqref{eq:cont_lang} and then deriving the evolution of the KL divergence of the samples created by~\eqref{eq:avg_dyn_2}. {To that end, we use a similar approach to~\cite[Section S4, (S97)-(S104)]{parayil2020decentralized}. Let $t$ correspond to the continuous-time within the $k$-th cycle where $t=t_k^j$ ($0\leq j<T$) coincides with the $j$-th computation. In the following expressions, $\bar{\bm{w}}_k^t$ represents the continuous variable corresponding to $\bar{\bm{w}}$ within $t \in [t_k^j, t_k^{j+1})$ of $k$-th cycle, while $\bar{\bm{w}}_k^j$ is the fixed value of $\bar{\bm{w}}$ at $t=t_j$ of $k$-th cycle. Then the FP equation for~\eqref{eq:cont_lang} is given by} 
\begin{align}
    &\frac{\partial p(\bar{\bm{w}}_k^t)}{ \partial t} = \nabla_{\bar{\bm{w}}_k^t} \cdot \Bigg[p(\bar{\bm{w}}_k^t) \nabla \,\log\left(\frac{p(\bar{\bm{w}}_k^t)}{p^*} \right)\Bigg] \nonumber \\
    &\qquad - \nabla_{\bar{\bm{w}}_k^t} \cdot \Bigg[ \int \sum_{\mathbf{A}} \sum_{\mathbf{B}} p(\bar{\bm{w}}_k^t, \bar{\bm{w}}_k^j, \tilde{\mathbf{w}}_k^j|\mathcal{A}_k, \mathcal{B}_k) p(\mathcal{A}_k) \times \nonumber \\
    &\qquad p(\mathcal{B}_k) \Big( \nabla E(\bar{\bm{w}}_k^t) \nonumber - \nabla E(\bar{\bm{w}}_k^j) + \xi( \bar{\bm{w}}_k^j, \mathcal{A}_k) \nonumber \\
    &\qquad  - \Phi(\bar{\bm{w}}_k^j, \tilde{\mathbf{w}}_k^j, \mathcal{A}_k ) \nonumber - \zeta(\bar{\bm{w}}_k^j, \tilde{\mathbf{w}}_k^j, \mathcal{A}_k, \mathcal{B}_k) \Big) d\bar{\bm{w}}_k^j d\mathbf{w}_k^j \Bigg], \nonumber \\
    &\,\, = \nabla_{\bar{\bm{w}}_k^t} \cdot \Bigg[p(\bar{\bm{w}}_k^t) \nabla \log\left(\frac{p(\bar{\bm{w}}_k^t)}{p^*} \right)\Bigg] + \nabla_{\bar{\bm{w}}_k^t} \cdot \tilde{f}(\bar{\bm{w}}_k^t), \label{eq:FP}
\end{align}
where $\tilde{f}(\bar{\bm{w}}_k^t) = \int \sum_{\mathbf{A}} \sum_{\mathbf{B}} p(\bar{\bm{w}}_k^t, \bar{\bm{w}}_k^j, \tilde{\mathbf{w}}_k^j| \mathcal{A}_k) p(\mathcal{A}_k) \Big( \nabla E(\bar{\bm{w}}_k^t) - \nabla E(\bar{\bm{w}}_k^j) - \Phi(\bar{\bm{w}}_k^j, \tilde{\mathbf{w}}_k^j, \mathcal{A}_k ) \Big) d\bar{\bm{w}}_k^j d \tilde{\mathbf{w}}_k^j$. Making use of the fact that $\mathbb{E}_{p(\mathcal{A}_k)} [\xi( \bar{\bm{w}}_k^j, \mathcal{A}_k)] = 0$ and $\mathbb{E}_{p(\mathcal{B}_k)} [ \zeta(\bar{\bm{w}}_k^j, \tilde{\mathbf{w}}_k^j, \mathcal{A}_k, \mathcal{B}_k)] = 0$, the terms $\xi( \bar{\bm{w}}_k^j, \mathcal{A}_k)$ and $\zeta(\bar{\bm{w}}_k^j, \tilde{\mathbf{w}}_k^j, \mathcal{A}_k, \mathcal{B}_k)$ terms vanish from $\tilde{f}(\bar{\bm{w}}_k^t)$ when the summations over $\mathbf{A}$ and $\mathbf{B}$ are performed respectively. Next, from definition of KL divergence in~\eqref{eq:KL_div_def}, we have 
\begin{align} \label{eq:KL_div_der}
    \dot{F}\left( p(\bar{\bm{w}}_k^t) \right) = \int\left( \log\left(\frac{p(\bar{\bm{w}}_k^t)}{p^*} \right) + 1\right) \frac{\partial p(\bar{\bm{w}}_k^t)}{\partial t} d\bar{\bm{w}}_k^t.
\end{align}
Substituting~\eqref{eq:FP} in~\eqref{eq:KL_div_der} results in 
\begin{align} \label{eq:Conv_1}
\begin{split}
    &\dot{F}\left( p(\bar{\bm{w}}_k^t) \right) = - \mathbb{E}_{p} \Bigg\|\nabla \log\left(\frac{p(\bar{\bm{w}}_k^t)}{p^*} \right)\Bigg\|^2 \\
    &\qquad + \underbrace{\int \nabla \log \left(\frac{p(\bar{\bm{w}}_k^t)}{p^*} \right)^\top \tilde{f}(\bar{\bm{w}}_k^t) d\bar{\bm{w}}_k^t}. 
\end{split} \\
    &\qquad \qquad \qquad \qquad \quad \, \Theta \nonumber
\end{align}

We  analyze $\Theta$ in~\eqref{eq:Conv_1}. Substituting $\tilde{f}(\bar{\bm{w}}_k^t)$ in~\eqref{eq:Conv_1} yields
\begin{align} \label{eq:Conv_2}
    \Theta &= \iiint \sum_{\mathbf{A}} \sum_{\mathbf{B}} \nabla \log \left(\frac{p(\bar{\bm{w}}_k^t)}{p^*} \right)^\top \Big( \nabla E(\bar{\bm{w}}_k^t) - \nabla E(\bar{\bm{w}}_k^j) \nonumber \\
    &\quad - \Phi(\bar{\bm{w}}_k^j, \tilde{\mathbf{w}}_k^j\mathcal{A}_k) \Big) p(\bar{\bm{w}}_k^t, \bar{\bm{w}}_k^j, \tilde{\mathbf{w}}_k^j|\mathcal{A}_k, \mathcal{B}_k) p(\mathcal{A}_k) p(\mathcal{B}_k) \nonumber \\
    &\qquad \qquad  d\bar{\bm{w}}_k^t d\bar{\bm{w}}_k^j d\tilde{\mathbf{w}}_k^j, \nonumber \\
    & \leq \frac{1}{2} \mathbb{E}_{p} \Bigg\|\nabla \log\left(\frac{p(\bar{\bm{w}}_k^t)}{p^*} \right)\Bigg\|^2 + \bar{L}^2 \iiint \sum_{\mathbf{A}} \sum_{\mathbf{B}}\|\bar{\bm{w}}_k^t- \bar{\bm{w}}_k^j\|^2 \times \nonumber \\
    &\qquad \qquad p(\bar{\bm{w}}_k^t, \bar{\bm{w}}_k^j, \tilde{\mathbf{w}}_k^j|\mathcal{A}_k, \mathcal{B}_k) p(\mathcal{A}_k) p(\mathcal{B}_k) d\bar{\bm{w}}_k^t d\bar{\bm{w}}_k^j d\tilde{\mathbf{w}}_k^j \nonumber \\
    &\quad + \iiint \sum_{\mathbf{A}} \|\Phi(\bar{\bm{w}}_k^j, \tilde{\mathbf{w}}_k^j, \mathcal{A}_k) \|^2 p(\bar{\bm{w}}_k^t, \bar{\bm{w}}_k^j, \tilde{\mathbf{w}}_k^j|\mathcal{A}_k) p(\mathcal{A}_k) \nonumber \\
    &\qquad \qquad d\bar{\bm{w}}_k^t d\bar{\bm{w}}_k^j d\tilde{\mathbf{w}}_k^j.
\end{align}
{Next, we analyze the individual norms in~\eqref{eq:Conv_2}. The last term in~\eqref{eq:Conv_2} satisfies 
\begin{align}
    &\iiint \sum_{\mathbf{A}} \|\Phi(\bar{\bm{w}}_k^j, \tilde{\mathbf{w}}_k^j\mathcal{A}_k) \|^2 p(\bar{\bm{w}}_k^t, \bar{\bm{w}}_k^j, \tilde{\mathbf{w}}_k^j|\mathcal{A}_k) p(\mathcal{A}_k) \nonumber \\
    &d\bar{\bm{w}}_k^t d\bar{\bm{w}}_k^j d\tilde{\mathbf{w}}_k^j \nonumber \\
    &\quad \leq \iint \sum_{\mathbf{A}} \left\| \sum_{i\in \mathcal{A}_k} \frac{1}{p_i} (\nabla E_i(\bm{w}_{i,k}^j) - \nabla E_i(\bar{\bm{w}}_k^j) ) \right\|^2 p(\mathcal{A}_k) \nonumber \\
    &\qquad p(\bar{\bm{w}}_k^j, \tilde{\mathbf{w}}_k^j) \left(\int p(\bar{\bm{w}}_k^t|\mathcal{A}_k) d\bar{\bm{w}}_k^t \right) d\bar{\bm{w}}_k^j d\tilde{\mathbf{w}}_k^j, \\
    &\quad \leq \frac{2}{p_m^2} \iint \left(\sum_{\mathbf{A}} \sum_{i\in \mathcal{A}_k} \left\| \nabla E_i(\bm{w}_{i,k}^j) - \nabla E_i(\bar{\bm{w}}_k^j) \right\|^2 p(\mathcal{A}_k) \right) \nonumber \\
    &\qquad \qquad p(\bar{\bm{w}}_k^j, \tilde{\mathbf{w}}_k^j) d\bar{\bm{w}}_k^j d\tilde{\mathbf{w}}_k^j, \\
    &\quad \leq \frac{2}{p_m^2} \frac{n(n-1)}{2} L^2 \iint \sum_{i=1}^n \| \bm{w}_{i,k}^j - \bar{\bm{w}}_k^j \|^2 p(\bar{\bm{w}}_k^j, \tilde{\mathbf{w}}_k^j) d\bar{\bm{w}}_k^j d\tilde{\mathbf{w}}_k^j, \nonumber \\
    &\quad \leq \frac{n(n-1)L^2}{p_m^2} \int \|\tilde{\mathbf{w}}_k^j\|^2 p(\tilde{\mathbf{w}}_k^j) d\tilde{\mathbf{w}}_k^j, \\
    &\quad \leq \frac{n(n-1)L^2}{p_m^2} \mathbb{E}[\|\tilde{\mathbf{w}}_k\|^2]. \label{eq:Conv_3}
\end{align}
Next, we look into $\|\bar{w}_k^t - \bar{w}_k^j\|$. Integrating~\eqref{eq:cont_lang} from $t=t_k^j$ to some $t< t_k^{j+1}$ yields 
\begin{align}
    \bar{\bm{w}}_k^t - \bar{\bm{w}}_k^j = -\widetilde{\nabla E}_k (t - t_k^j) + \sqrt{2} \Big(\bm{B}(t) - \bm{B}(t_k^j) \Big). \label{eq:Conv_3.0}
\end{align}
Substituting~\eqref{eq:grad_terms} in~\eqref{eq:Conv_3.0} and thereafter taking the norm yield
\begin{align}
    &\|\bar{\bm{w}}_k^t - \bar{\bm{w}}_k^j\|^2 \leq \Big\| (t-t_k^j) \Big(-\nabla E(\bar{\bm{w}}_k^j) + \xi(\bar{\bm{w}}_k^j, \tilde{\mathbf{w}}_k^j) - \Phi(\bar{\bm{w}}_k^j, \nonumber \\
    & \tilde{\mathbf{w}}_k^j, \mathcal{A}_k) - \zeta(\bar{\bm{w}}_k^j, \tilde{\mathbf{w}}_k^j, \mathcal{A}_k, \mathcal{B}_k) \Big) + \sqrt{2} \big(\bm{B}(t) - \bm{B}(t_k^j) \big) \Big\|^2 \nonumber \\
    &\leq (t-t_k^j)^2 \Big\|-\nabla E(\bar{\bm{w}}_k^j) + \xi(\bar{\bm{w}}_k^j, \tilde{\mathbf{w}}_k^j) - \Phi(\bar{\bm{w}}_k^j,\tilde{\mathbf{w}}_k^j, \mathcal{A}_k)\nonumber \\
    &\quad - \zeta(\bar{\bm{w}}_k^j, \tilde{\mathbf{w}}_k^j, \mathcal{A}_k, \mathcal{B}_k) \Big\|^2 + 2\Big\|\bm{B}(t) - \bm{B}(t_k^j) \Big\|^2 \nonumber \\
    &\quad + 2\sqrt{2} (t-t_k^j) \Big(\bm{B}(t) - \bm{B}(t_k^j) \Big)^\top \Big(-\nabla E(\bar{\bm{w}}_k^j) + \xi(\bar{\bm{w}}_k^j, \nonumber \\
    &\qquad \quad \tilde{\mathbf{w}}_k^j) -\Phi(\bar{\bm{w}}_k^j, \tilde{\mathbf{w}}_k^j, \mathcal{A}_k) - \zeta(\bar{\bm{w}}_k^j, \tilde{\mathbf{w}}_k^j, \mathcal{A}_k, \mathcal{B}_k) \Big), \\
    &\leq 4\alpha_k^2 \Big( \|\nabla E(\bar{\bm{w}}_k^j)\|^2 + \|\xi(\bar{\bm{w}}_k^j, \tilde{\mathbf{w}}_k^j)\|^2 + \| \Phi(\bar{\bm{w}}_k^j,\tilde{\mathbf{w}}_k^j, \mathcal{A}_k) \|^2 \nonumber \\
    &\quad + \|\zeta(\bar{\bm{w}}_k^j, \tilde{\mathbf{w}}_k^j, \mathcal{A}_k, \mathcal{B}_k)\|^2 \Big) + 2\Big\|\bm{B}(t) - \bm{B}(t_k^j) \Big\|^2 \nonumber \\
    &\quad + 2\sqrt{2} (t-t_k^j) \Big(\bm{B}(t) - \bm{B}(t_k^j) \Big)^\top \Big(-\nabla E(\bar{\bm{w}}_k^j) + \xi(\bar{\bm{w}}_k^j, \nonumber \\
    &\qquad \quad \tilde{\mathbf{w}}_k^j) -\Phi(\bar{\bm{w}}_k^j, \tilde{\mathbf{w}}_k^j, \mathcal{A}_k) - \zeta(\bar{\bm{w}}_k^j, \tilde{\mathbf{w}}_k^j, \mathcal{A}_k, \mathcal{B}_k) \Big), \label{eq:Conv_3.1}
\end{align}
where we used $t-t_k^j \leq \alpha_k$. {Without loss of generality, assuming $\nabla E(\mathbf{0}) = \mathbf{0}$, we have $\|\nabla E(\bar{\bm{w}}_k^j) \|^2 \leq \bar{L}^2 \| \bar{\bm{w}}_k^j\|^2$.} Then, 
\begin{align}
    &\iiint \sum_{\mathbf{A}} \| \nabla E(\bar{\bm{w}}_k^j) \|^2 p(\bar{\bm{w}}_k^t, \bar{\bm{w}}_k^j, \tilde{\mathbf{w}}_k^j,\mathcal{A}_k) d\bar{\bm{w}}_k^t d\bar{\bm{w}}_k^j d\tilde{\mathbf{w}}_k^j \nonumber \\ 
    & \qquad \qquad \qquad \leq \bar{L}^2 C_{_{\bar{\bm{w}}}}. \label{eq:Conv_3.2}
\end{align}
Furthermore, using~\eqref{eq:assump_3_1} and~\eqref{eq:assump_3_2} of Assumption~\ref{assump:sto_grad}, the following bounds can be respectively derived.
\begin{align}
    &\iiint \sum_{\mathbf{A}} \|\xi(\bar{\bm{w}}_k^j, \tilde{\mathbf{w}}_k^j)\|^2 p(\bar{\bm{w}}_k^t, \bar{\bm{w}}_k^j, \tilde{\mathbf{w}}_k^j,\mathcal{A}_k)  \nonumber \\ 
    & \qquad \qquad \qquad d\bar{\bm{w}}_k^t d\bar{\bm{w}}_k^j d\tilde{\mathbf{w}}_k^j \leq C_\xi, \label{eq:Conv_3.3}
\end{align}
and 
\begin{align}
    &\iiint \sum_{\mathbf{A}} \sum_{\mathbf{B}} \|\zeta(\bar{\bm{w}}_k^j, \tilde{\mathbf{w}}_k^j, \mathcal{A}_k, \mathcal{B}_k)\|^2 p(\bar{\bm{w}}_k^t, \bar{\bm{w}}_k^j, \tilde{\mathbf{w}}_k^j, \mathcal{A}_k, \mathcal{B}_k) \nonumber \\ 
    & \qquad \qquad \qquad d\bar{\bm{w}}_k^t d\bar{\bm{w}}_k^j d\tilde{\mathbf{w}}_k^j \leq C_\zeta. \label{eq:Conv_3.4}
\end{align}
Next, 
\begin{align}
    &\iiint \sum_{\mathbf{A}} \Big\|\bm{B}(t) - \bm{B}(t_k^j) \Big\|^2 p(\bar{\bm{w}}_k^t, \bar{\bm{w}}_k^j, \tilde{\mathbf{w}}_k^j, \mathcal{A}_k) \nonumber \\ 
    & \qquad \qquad  d\bar{\bm{w}}_k^t d\bar{\bm{w}}_k^j d\tilde{\mathbf{w}}_k^j \leq \alpha_k d_w. \label{eq:Conv_3.5}
\end{align}
For details on~\eqref{eq:Conv_3.2}, refer to~\cite[S(135)-S(137)]{parayil2020decentralized}. Finally, the last term in~\eqref{eq:Conv_3.1}, on taking the expectations, vanishes. Refer to~\cite[(S139) - (S141)]{parayil2020decentralized}. Thus, substituting~\eqref{eq:Conv_3.2}-\eqref{eq:Conv_3.5} in~\eqref{eq:Conv_3.1} gives
\begin{align}
    &\bar{L}^2 \iiint \|\bar{\bm{w}}_k^t- \bar{\bm{w}}_k^j\|^2 p(\bar{\bm{w}}_k^t, \bar{\bm{w}}_k^j, \tilde{\mathbf{w}}_k^j) d\bar{\bm{w}}_k^t d\bar{\bm{w}}_k^j d\tilde{\mathbf{w}}_k^j \leq 4\alpha_k^2 \bar{L}^2 \times \nonumber \\
    &\, \left[ \bar{L}^2 C_{_{\bar{\bm{w}}}} + C_\xi + C_\zeta + \frac{n(n-1)\bar{L}^2}{p_m^2} \mathbb{E}[\| \tilde{\mathbf{w}}_k\|^2] \right] + 2\alpha_k \bar{L}^2 d_w. \label{eq:Conv_4}
\end{align}
}

%Further analysis of the terms of $\Theta$  results in 
%\begin{align} \label{eq:Conv_3}
%\begin{split}
%    &\iiint \sum_{\mathbf{A}} \|\Phi(\bar{\bm{w}}_k^j, \tilde{\mathbf{w}}_k^j\mathcal{A}_k) \|^2 p(\bar{\bm{w}}_k^t, \bar{\bm{w}}_k^j, \tilde{\mathbf{w}}_k^j|\mathcal{A}_k) \times \\
%    &\qquad p(\mathcal{A}_k) d\bar{\bm{w}}_k^t d\bar{\bm{w}}_k^j d\tilde{\mathbf{w}}_k^j \leq \frac{n(n-1)L^2}{p_m^2} \mathbb{E}[\|\tilde{\mathbf{w}}_k\|^2],
%\end{split}
%\end{align}

Using the results from~\eqref{eq:Conv_3} and~\eqref{eq:Conv_4} in~\eqref{eq:Conv_2} and substituting that in~\eqref{eq:Conv_1} yield
\begin{align} \label{eq:Conv_5}
\begin{split}
    &\dot{F}\left( p(\bar{\bm{w}}_k^t) \right) \leq -\frac{1}{2} \mathbb{E}_{p( \bar{\bm{w}}_k^t)} \Bigg\|\nabla \log\left(\frac{p(\bar{\bm{w}}_k^t)}{p^*} \right)\Bigg\|^2 \\
    &\qquad + 4\alpha_k^2 \bar{L}^2 \left[ \bar{L}^2 C_{_{\bar{\bm{w}}}} + C_{_\xi} + C_{_\zeta} \right] + 2\alpha_k\bar{L}^2d_w \\
    &\qquad + \frac{n(n-1) (\alpha_k^2 \bar{L}^2 + 1)L^2}{p_m^2} \mathbb{E}[\|\tilde{\mathbf{w}}_k\|^2]. 
\end{split}
\end{align}
Using~\eqref{eq:LSI} in~\eqref{eq:Conv_5} gives
\begin{align} \label{eq:Conv_6}
    &\dot{F}\left( p(\bar{\bm{w}}_k^t) \right) \leq -\rho_U F\left( p(\bar{\bm{w}}_k^t) \right) + 4\alpha_k^2 \bar{L}^2  \left[\bar{L}^2 C_{_{\bar{\bm{w}}}} + C_{_\xi} + C_{_\zeta} \right] \nonumber \\
    &\qquad + 2\alpha_k\bar{L}^2d_w + \frac{n(n-1) (\alpha_k^2 \bar{L}^2 + 1)L^2}{p_m^2} \mathbb{E}[\|\tilde{\mathbf{w}}_k\|^2].
\end{align}
Integrating~\eqref{eq:Conv_6} w.r.t. $t$ from $t_k^j$ to $t_k^{j+1}$ and using the result $\frac{1-\exp(-\rho_U(t_k^{j+1}-t_k^j))}{\rho_U} \leq t_k^{j+1} - t_k^j$ and $t_k^{j+1} - t_k^j = \alpha_k$, we have
\begin{align} \label{eq:Conv_7}
    &{F}\left( p(\bar{\bm{w}}_k^{j+1}) \right) \leq \exp(-\rho_U \alpha_k) F\left( p(\bar{\bm{w}}_k^j) \right) + \eta_k,
\end{align}
where $\eta_k = 4 \alpha_k^3 \bar{L}^2  \big[\bar{L}^2 C_{_{\bar{\bm{w}}}} + C_{_\xi} + C_{_\zeta} \big] + 2\alpha_k^2 \bar{L}^2d_w + \frac{n(n-1) (\alpha_k^2 \bar{L}^2 + 1) \alpha_k L^2}{p_m^2} \mathbb{E}[\|\tilde{\mathbf{w}}_k\|^2]$.
Thus,~\eqref{eq:Conv_7} gives the evolution of KL divergence of the distribution of the average of samples between each computation of a cycle. Iteratively using~\eqref{eq:Conv_7} from $j=0$ to $T$ and noting that $\bar{\bm{w}}_k^0 = \bar{\bm{w}}_k$ and $\bar{\bm{w}}_k^T = \bar{\bm{w}}_{k+1}$ gives the evolution of the KL divergence between successive cycles:
\begin{align} \label{eq:Conv_8}
    {F}\left( p (\bar{\bm{w}}_{k+1}) \right) &\leq \exp(-\rho_U \alpha_k T) F\left( p(\bar{\bm{w}}_k) \right) + T\eta_k.
\end{align}
%Note that the expression in~\eqref{eq:Conv_8} corresponds to~\cite[Section S4, (S166)]{parayil2020decentralized}. 
{Iteratively using~\eqref{eq:Conv_8} further results in }
\begin{align} \label{eq:Conv_9}
\begin{split}
    {F}\left( p (\bar{\bm{w}}_{k+1}) \right) &\leq F\left( p(\bar{\bm{w}}_0) \right) \exp \left( -\rho_U T \sum_{\ell=0}^{k} \alpha_\ell \right) \\
    &\quad + T \sum_{\ell=0}^{k} \eta_\ell \exp \left( -\rho_U T \sum_{i=\ell+1}^{k} \alpha_i \right).
\end{split}
\end{align}
%Using (S168)-(S170) of~\cite{parayil2020decentralized} yields
{From Lemma~\ref{lemma:summ_bounds}, we have 
\begin{align}
    \sum_{\ell=0}^{k} \alpha_\ell \geq \int_{0}^{k} \frac{a}{(\ell+1)^{\delta_\alpha}} d\ell = \frac{a(k+1)^{\delta_\alpha}}{1 - \delta_\alpha} - \frac{a}{1-\delta_\alpha},
\end{align}
which when applied to~\eqref{eq:Conv_8} results in}
\begin{align} \label{eq:Conv_10}
    &{F}\left( p (\bar{\bm{w}}_{k+1}) \right) \leq  F\left( p(\bar{\bm{w}}_0) \right) \exp\left( \frac{a\rho_UT}{1-\delta_\alpha}\right) \exp \bigg( -\frac{a\rho_UT}{1-\delta_\alpha} \times \nonumber \\ 
    &\qquad (k+1)^{1-\delta_\alpha} \bigg) + T \sum_{\ell=0}^{k} \eta_\ell \exp \left( -\rho_UT\sum_{i=\ell+1}^{k} \alpha_i \right),
\end{align}
where $\eta_k$ is bounded as 
\begin{align} \label{eq:Conv_11}
    \eta_k \leq \frac{C_\eta}{k^{2\delta_\alpha}} + \bar{W}_1 \exp \left( -\left|\ln \bar{\lambda} \right|k \right),
\end{align}
in which $C_\eta = 4a^3\bar{L}^2 \big[\bar{L}^2 C_{_{\bar{\bm{w}}}} + C_{_\xi} + C_{_\zeta} \big] + 2a^2 \bar{L}^2d_w + \frac{n(n-1) (a^2 \bar{L}^2 + 1) aL^2 W_2}{p_m^2}$ and $\bar{W}_1 = \frac{n(n-1) (a^2 \bar{L}^2 + 1) aL^2W_1}{p_m^2}$. Thus,~\eqref{eq:Conv_10} and~\eqref{eq:Conv_11} correspond to (S171) and (S164) in Section S4 of~\cite{parayil2020decentralized}. Thereafter, following a similar procedure to that section leads to the result in~\eqref{eq:convergence_rate} of Theorem~\ref{thm:Convergence}.

\subsection{Useful Lemmas}\label{sec:lemma}
\begin{lemma} \label{lemma:summ_bounds}
For any $k \in \mathbb{N}$, let $f(k)$ be a non-negative function. Then, for some $0<t<T$, if $f(k)$ is a decreasing function, we have the following result.
\begin{align}
    \int_{t}^{T} f(x)dx \leq \sum_{k=t}^{T} f(k) \leq \int_{t-1}^{T} f(x)dx. 
\end{align}
Similarly, $f(k)$ is increasing function, the corresponding result holds.
\begin{align}
    \int_{t-1}^{T} f(x)dx \leq \sum_{k=t}^{T} f(k) \leq \int_{t}^{T+1} f(x)dx. 
\end{align}
\end{lemma}
See Appendix A2 in~\cite{cormen2022introduction}.

\begin{lemma} \label{lemma:consensus_rate_lemma}
For any $k \in \mathbb{N}$, let $\{x_k\}$ be a non-negative sequence satisfying the following relation.
\begin{align} \label{eq:lemma_seq}
    &x_{k+1} \leq \sigma_0 x_k + \frac{\mu_x}{(k+1)^\delta}, 
\end{align}
where $\sigma_0 \in (0,1)$ and $\mu_x, \delta>0$ are constants. Then the sequence converges with the rate below.
\begin{align} \label{eq:lemma_final_form}
    &x_{k+1} \leq X_1 \sigma_0^{k+1} + \frac{X_2}{(k+1)^\delta},
\end{align}
where 
\begin{align}
    X_1 &= x_0 + \mu_x \sum_{t=0}^{\bar{t}-1} \frac{\sigma_0^{-(t+1)}}{(t+1)^{\delta}}, \\
    X_2 &= - \frac{\mu_x}{\sigma_0} \left( \ln{\sigma_0} + \frac{\delta}{\bar{t}+1} \right)^{-1}, \qquad \qquad  
    \text{and} \\
    \bar{t} &= \max\Bigg\{0,\ceil[\Bigg]{\frac{\delta}{|\ln{\sigma_0}|} - 1} \Bigg\}.
\end{align}
\end{lemma}

%\begin{proof}
{The proof follows a similar approach to~\cite[Lemma S4]{parayil2020decentralized}.  }

\begin{lemma} \label{lemma:bounded_avg_samples}
{There exists $C_{_{\bar{\bm{w}}}} > 0$ such that for any $k > 0$ and $0 \leq j \leq T$, we have $\mathbb{E}_{p(\bar{\bm{w}}_k) }[\| \bar{\bm{w}}_k\|^2] \leq C_{_{\bar{\bm{w}}}}$ and $\mathbb{E}_{p(\bar{\bm{w}}_k^j) }[\| \bar{\bm{w}}_k^j\|^2] \leq C_{_{\bar{\bm{w}}}}$.} \end{lemma}

\begin{proof}
{This proof uses the law of induction. Let us assume that that there exists some $C_{_{\bar{\bm{w}}}} > 0$ such that $\mathbb{E}_{p(\bar{\bm{w}}_\ell) }[\| \bar{\bm{w}}_\ell\|^2] \leq C_{_{\bar{\bm{w}}}}$ and $\mathbb{E}_{p(\bar{\bm{w}}_\ell^j) }[\| \bar{\bm{w}}_\ell^j\|^2] \leq C_{_{\bar{\bm{w}}}}$ for all $\ell \leq k$ and $0\leq j < T$. From~\eqref{eq:fusion}, we have $\bar{\bm{w}}_k^0 = \bar{\bm{w}}_k$, hence $\mathbb{E}_{p(\bar{\bm{w}}_k^0) }[\| \bar{\bm{w}}_k^0\|^2] \leq C_{_{\bar{\bm{w}}}}$. Next, we intend to prove that if $\mathbb{E}_{p(\bar{\bm{w}}_k^u) }[\| \bar{\bm{w}}_k^u\|^2] \leq C_{_{\bar{\bm{w}}}}$ for all $0\leq u \leq j$, then it implies $\mathbb{E}_{p(\bar{\bm{w}}_k^{j+1}) }[\| \bar{\bm{w}}_k^{j+1}\|^2] \leq C_{_{\bar{\bm{w}}}}$. To that end, we couple $\bar{\bm{w}}^*$ optimally with $\bar{\bm{w}}_{k}^{j+1} \sim p(\bar{\bm{w}}_{k}^{j+1})$, i.e., $(\bm{w}^*, \bar{\bm{w}}_{k}^{j+1}) \sim \gamma \in \Gamma_{opt}(p(\bar{\bm{w}}_{k}^{j+1}), p^*)$. This yields
\begin{align}
    &\mathbb{E}_{p(\bar{\bm{w}}_{k}^{j+1})}[\| \bar{\bm{w}}_{k}^{j+1} \|^2] = \mathbb{E}_{(\bar{\bm{w}}_{k}^{j+1}, \bar{\bm{w}}^*) \sim \gamma} [\| \bar{\bm{w}}_{k}^{j+1} - \bar{\bm{w}}^* +\bar{\bm{w}}^*\|^2], \nonumber \\
    &\quad \leq 2\mathbb{E}_{(\bar{\bm{w}}_{k}^{j+1}, \bar{\bm{w}}^*) \sim \gamma} [\| \bar{\bm{w}}_{k}^{j+1} - \bar{\bm{w}}^*\|^2] + 2\mathbb{E}_{\bar{\bm{w}}^* \sim p^*} [\|\bar{\bm{w}}^*\|^2] \nonumber \\
    &\quad \leq 2\mathcal{W}(p(\bar{\bm{w}}_{k}^{j+1}), p^*) + 2c_1 , \leq \frac{4}{\rho_U} F ( p(\bar{\bm{w}}_{k}^{j+1}) ) + 2c_1, \label{eq:exp_cons_err}
\end{align}
where $\mathcal{W}(\cdot, \cdot)$ is the Wasserstein distance between two distributions and $\mathbb{E}_{\bar{\bm{w}}^* \sim p^*} [\|\bar{\bm{w}}^*\|^2] \leq c_1$. With the assumption that $\mathbb{E}_{p(\bar{\bm{w}}_k^u) }[\| \bar{\bm{w}}_k^u\|^2] \leq C_{_{\bar{\bm{w}}}}$ for all $0\leq u \leq j$,~\eqref{eq:Conv_7} holds until $u=j$ and can rewritten as follows. 
\begin{align} \label{eq:Lemma3_1}
    &{F}( p(\bar{\bm{w}}_k^{j+1}) ) \leq \exp(-\rho_U \alpha_k) F( p(\bar{\bm{w}}_k^j) ) + \bar{\eta} + 4 a^3 \bar{L}^4 C_{_{\bar{\bm{w}}}},
\end{align}
where $\bar{\eta} = 4 a^3 \bar{L}^2  \big[C_{_\xi} + C_{_\zeta} \big] + 2a^2 \bar{L}^2d_w + \frac{n(n-1) (a^2 \bar{L}^2 + 1) a L^2}{p_m^2} \mathbb{E}[\|\tilde{\mathbf{w}}_k\|^2]$. {Since~\eqref{eq:con_syn}-\eqref{eq:cons_error_sq_1.2} hold until $\ell \leq k$,~\eqref{eq:con_err_syn_norm_sq_f} holds up to $\ell = k$. Applying Lemma~\ref{lemma:consensus_rate_lemma} for $0 \leq \ell \leq k$, we have 
\begin{align} 
    \mathbb{E}[\|\tilde{\mathbf{w}}_{k+1}\|^2] &\leq W'_1 \bar{\lambda}^{k+1} + \frac{W'_2}{(k+1)^{\delta_\alpha}}, \label{eq:int_cons_err}
\end{align}
where $W'_1 = \mathbb{E}[\| \tilde{\mathbf{w}}_{0}\|^2] + C_\mu \sum_{t=0}^{\bar{t}-1} \frac{\sigma_0^{-(t+1)}}{(t+1)^{\delta}}$, $\bar{t} = \max\Bigg\{0,\ceil[\Bigg]{\frac{\delta_\alpha}{|\ln{\bar{\lambda}}|} - 1} \Bigg\}$ and $W'_2 = - \frac{C_\mu}{\bar{\lambda}} \left( \ln{\bar{\lambda}} + \frac{\delta_\alpha}{\bar{t}+1} \right)^{-1}$.
Thus,~\eqref{eq:int_cons_err} implies $\bar{\eta} < \infty$. Using~\eqref{eq:Lemma3_1} iteratively backward up to $j=0$ and applying $\bar{\bm{w}}_k^0 = \bar{\bm{w}}_k$ further yield
\begin{align} 
    {F}( p(\bar{\bm{w}}_k^{j+1}) ) &\leq F( p(\bar{\bm{w}}_k) ) + \sum_{\ell=0}^{j} \Big(\bar{\eta} + 4 a^3 \bar{L}^4 C_{_{\bar{\bm{w}}}} \Big), \nonumber \\
    &\leq F( p(\bar{\bm{w}}_k) ) + T\bar{\eta} + 4a^3 T \bar{L}^4 C_{_{\bar{\bm{w}}}}. \label{eq:Lemma3_2}
\end{align}
Next, with the assumption that $\mathbb{E}_{p(\bar{\bm{w}}_\ell) }[\| \bar{\bm{w}}_\ell\|^2] \leq C_{_{\bar{\bm{w}}}}$ for $0\leq \ell \leq k$,~\eqref{eq:KL_div_evo_1cycle} applies for $0\leq \ell \leq k$ as:
\begin{align}
\begin{split}
   F\left(p(\bar{\bm{w}}_{k})\right) &\leq \exp(-\rho_U\alpha_{k-1} T) F\left( p(\bar{\bm{w}}_{k-1}) \right) + \bar{\omega}_{k-1} \\
   &\qquad + 4\alpha_{k-1}^3\bar{L}^4T C_{_{\bar{\bm{w}}}}, \label{eq:Lemma3_3}
\end{split}
\end{align}
where $\bar{\omega}_{k-1} = 4\alpha_{k-1}^3\bar{L}^2 T (C_\xi + C_\zeta) + 2\alpha_{k-1}^2\bar{L}^2 Td_w + \frac{n (n-1)T}{p_m^2} (4\alpha_{k-1}^3L^2 \bar{L}^2 + \alpha_{k-1} L^2) \mathbb{E}[ \|\tilde{\mathbf{w}}_{k-1}\|^2]$. Using~\eqref{eq:Lemma3_3} iteratively yields
\begin{align}
    F\left(p(\bar{\bm{w}}_{k})\right) &\leq F\left( p(\bar{\bm{w}}_0) \right) + \sum_{\ell=0}^{k-1} \bar{\omega}_\ell + 4\bar{L}^4T C_{_{\bar{\bm{w}}}} \sum_{\ell=0}^{k-1} \alpha_\ell^3. \label{eq:Lemma3_4}
\end{align}
Again, from~\eqref{eq:int_cons_err}, for $0\leq \ell < k$ we have 
\begin{align}
    \sum_{\ell=0}^{k-1} \bar{\omega}_\ell &\leq k_{\omega_1} \sum_{\ell=0}^{\infty} \alpha_\ell^3 + k_{\omega_2} \sum_{\ell=0}^{\infty} \alpha_\ell^2 + k_{\omega_3} W'_1 \sum_{\ell=0}^{\infty} \bar{\lambda}^{\ell+1} \nonumber \\
    &\qquad \quad + k_{\omega_4} W'_2 \sum_{\ell=0}^{\infty} \frac{1}{(k+1)^{2 \delta_\alpha}}, \nonumber \\
    &\leq \frac{3\delta_\alpha k_{\omega_1}}{3 \delta_\alpha - 1} + \frac{2\delta_\alpha (k_{\omega_2} + k_{\omega_3} W'_2)}{2\delta_\alpha - 1} + \frac{k_{\omega_3} \bar{\lambda} W'_1}{1- \bar{\lambda}} = c_2.
\end{align}}
where $k_{\omega_1} = 4\bar{L}^2 T (C_\xi + C_\zeta)$, $k_{\omega_2} = 2\bar{L}^2 Td_w$ and $k_{\omega_3} = \frac{n (n-1)T}{p_m^2} (4a^3L^2 \bar{L}^2 + aL^2) $.
%Since $2\delta_\alpha > 1$, $\bar{\omega}_\ell$ is an infinitely summable sequence, so let $\sum_{\ell=0}^\infty \bar{\omega}_\ell \leq c_2$. 
Also, $\sum_{\ell=0}^\infty \alpha_\ell^3 \leq \frac{3a^3 \delta_\alpha}{3\delta_\alpha-1}$. Substituting these bounds in~\eqref{eq:Lemma3_4} yields
\begin{align}
\begin{split}
    F\left(p(\bar{\bm{w}}_{k})\right) &\leq F\left( p(\bar{\bm{w}}_0) \right) + c_2 + \frac{12a^3 \delta_\alpha\bar{L}^4T}{3\delta_\alpha - 1} C_{_{\bar{\bm{w}}}} \\
    &\leq c_3 + \frac{12a^3 \delta_\alpha\bar{L}^4T}{3\delta_\alpha - 1} C_{_{\bar{\bm{w}}}}, \label{eq:Lemma3_5}
\end{split}
\end{align}
where $c_3 = c_2 + F\left( p(\bar{\bm{w}}_0) \right)$. Substituting~\eqref{eq:Lemma3_5} in~\eqref{eq:Lemma3_2} gives
\begin{align} 
    {F}( p(\bar{\bm{w}}_k^{j+1}) ) &\leq c_3 + T\bar{\eta} + 4a^3 T \bar{L}^4 \left( \frac{6\delta_\alpha-1}{3\delta_\alpha-1}\right) C_{_{\bar{\bm{w}}}}. \label{eq:Lemma3_6}
\end{align}
Substituting~\eqref{eq:Lemma3_6} in~\eqref{eq:exp_cons_err} gives 
\begin{align}
&\mathbb{E}_{p(\bar{\bm{w}}_{k}^{j+1})}[\| \bar{\bm{w}}_{k}^{j+1} \|^2] \leq c_4 + \frac{16a^3 T \bar{L}^4(6\delta_\alpha-1)}{\rho_U (3\delta_\alpha-1)} C_{_{\bar{\bm{w}}}}, \label{eq:exp_cons_err_2}
\end{align}
where $c_4 = 2c_1 + \frac{4}{\rho_U}(c_3 + T \bar{\eta})$. Let $\mathbb{E}_{p(\bar{\bm{w}}_{k}^{j+1})}[\| \bar{\bm{w}}_{k}^{j+1} \|^2] \leq C_{_{\bar{\bm{w}}}}$. This leads to the following bound.
\begin{align}
    C_{_{\bar{\bm{w}}}} &\geq \frac{c_4}{1 - \frac{16a^3 T \bar{L}^4(6\delta_\alpha-1)}{\rho_U (3\delta_\alpha-1)}}. \label{eq:Cw_bound}
\end{align}
Thus, there exists $C_{_{\bar{\bm{w}}}}>0$ such that $\mathbb{E}_{p(\bar{\bm{w}}_{k}^{j+1})}[\| \bar{\bm{w}}_{k}^{j+1} \|^2] \leq C_{_{\bar{\bm{w}}}}$ if~\eqref{eq:Cw_bound} is satisfied. Condition~\ref{cond:3} ensures that the bound~\eqref{eq:Cw_bound} exists. By extension we have $\mathbb{E}_{p(\bar{\bm{w}}_{k}^{T})}[\| \bar{\bm{w}}_{k}^{T} \|^2] \leq C_{_{\bar{\bm{w}}}}$. Further, from~\eqref{eq:final_update}, since $\bar{\bm{w}}_{k+1} \equiv \bar{\bm{w}}_{k}^{T}$, we have $\mathbb{E}_{p(\bar{\bm{w}}_{k+1})}[\| \bar{\bm{w}}_{k+1} \|^2] \leq C_{_{\bar{\bm{w}}}}$. Thus, choosing $C_{_{\bar{\bm{w}}}}$ as in~\eqref{eq:Cw_exp} satisfies $\mathbb{E}_{p(\bar{w}_{k})} [\|\bar{\bm{w}}_k\|^2] \leq C_{_{\bar{\bm{w}}}}$ and $\mathbb{E}_{p(\bar{w}_{k}^{j})} [\|\bar{\bm{w}}_k^{j}\|^2] \leq C_{_{\bar{\bm{w}}}}$ for any $k>0$ and $0\leq j \leq T$.
\begin{align}
    C_{_{\bar{\bm{w}}}} = \max\left\{ \mathbb{E}_{p(\bar{\bm{w}}_0)} [\|\bar{\bm{w}}_0\|^2 ], \frac{c_4}{1 - \frac{16a^3 T \bar{L}^4(6\delta_\alpha-1)}{\rho_U (3\delta_\alpha-1)}} \right\}.
\end{align} \label{eq:Cw_exp}
} 
\end{proof}

%% file: Main.bbl
\begin{thebibliography}{10}
\providecommand{\url}[1]{#1}
\csname url@samestyle\endcsname
\providecommand{\newblock}{\relax}
\providecommand{\bibinfo}[2]{#2}
\providecommand{\BIBentrySTDinterwordspacing}{\spaceskip=0pt\relax}
\providecommand{\BIBentryALTinterwordstretchfactor}{4}
\providecommand{\BIBentryALTinterwordspacing}{\spaceskip=\fontdimen2\font plus
\BIBentryALTinterwordstretchfactor\fontdimen3\font minus \fontdimen4\font\relax}
\providecommand{\BIBforeignlanguage}[2]{{%
\expandafter\ifx\csname l@#1\endcsname\relax
\typeout{** WARNING: IEEEtran.bst: No hyphenation pattern has been}%
\typeout{** loaded for the language `#1'. Using the pattern for}%
\typeout{** the default language instead.}%
\else
\language=\csname l@#1\endcsname
\fi
#2}}
\providecommand{\BIBdecl}{\relax}
\BIBdecl

\bibitem{parayil2020decentralized}
A.~Parayil, H.~Bai, J.~George, and P.~Gurram, ``Decentralized langevin dynamics for bayesian learning,'' \emph{Advances in Neural Information Processing Systems}, vol.~33, pp. 15\,978--15\,989, 2020.

\bibitem{lamport2019time}
L.~Lamport, ``Time, clocks, and the ordering of events in a distributed system,'' in \emph{Concurrency: the Works of Leslie Lamport}, 2019, pp. 179--196.

\bibitem{1104412}
J.~Tsitsiklis, D.~Bertsekas, and M.~Athans, ``Distributed asynchronous deterministic and stochastic gradient optimization algorithms,'' \emph{IEEE Transactions on Automatic Control}, vol.~31, no.~9, pp. 803--812, 1986.

\bibitem{5719290}
K.~Srivastava and A.~Nedic, ``Distributed asynchronous constrained stochastic optimization,'' \emph{IEEE Journal of Selected Topics in Signal Processing}, vol.~5, no.~4, pp. 772--790, 2011.

\bibitem{9217472}
B.~M. Assran, A.~Aytekin, H.~R. Feyzmahdavian, M.~Johansson, and M.~G. Rabbat, ``Advances in asynchronous parallel and distributed optimization,'' \emph{Proceedings of the IEEE}, vol. 108, no.~11, pp. 2013--2031, 2020.

\bibitem{tsitsiklis1984problems}
J.~N. Tsitsiklis, ``Problems in decentralized decision making and computation.'' Massachusetts Inst of Tech Cambridge Lab for Information and Decision Systems, Tech. Rep., 1984.

\bibitem{berahas2018balancing}
A.~S. Berahas, R.~Bollapragada, N.~S. Keskar, and E.~Wei, ``Balancing communication and computation in distributed optimization,'' \emph{IEEE Transactions on Automatic Control}, vol.~64, no.~8, pp. 3141--3155, 2018.

\bibitem{nedic2018network}
A.~Nedi{\'c}, A.~Olshevsky, and M.~G. Rabbat, ``Network topology and communication-computation tradeoffs in decentralized optimization,'' \emph{Proceedings of the IEEE}, vol. 106, no.~5, pp. 953--976, 2018.

\bibitem{Gutman04}
W.~X. I.~Gutman, ``Generalized inverse of the \textsc{L}aplacian matrix and some applications,'' \emph{Bulletin, Classe des Sciences Math\'{e}matiques et Naturelles, Sciences math\'{e}matiques}, vol. 129, no.~29, pp. 15--23, 2004.

\bibitem{dalalyan2017theoretical}
A.~S. Dalalyan, ``Theoretical guarantees for approximate sampling from smooth and log-concave densities,'' \emph{Journal of the Royal Statistical Society: Series B (Statistical Methodology)}, vol.~79, no.~3, pp. 651--676, 2017.

\bibitem{dalalyan2017further}
A.~Dalalyan, ``Further and stronger analogy between sampling and optimization: Langevin monte carlo and gradient descent,'' in \emph{Conference on Learning Theory}.\hskip 1em plus 0.5em minus 0.4em\relax PMLR, 2017, pp. 678--689.

\bibitem{cheng2018underdamped}
X.~Cheng, N.~S. Chatterji, P.~L. Bartlett, and M.~I. Jordan, ``Underdamped langevin mcmc: A non-asymptotic analysis,'' in \emph{Conference on learning theory}.\hskip 1em plus 0.5em minus 0.4em\relax PMLR, 2018, pp. 300--323.

\bibitem{cheng2018convergence}
X.~Cheng and P.~Bartlett, ``Convergence of langevin mcmc in kl-divergence,'' in \emph{Algorithmic Learning Theory}.\hskip 1em plus 0.5em minus 0.4em\relax PMLR, 2018, pp. 186--211.

\bibitem{durmus2016sampling}
A.~Durmus and E.~Moulines, ``Sampling from strongly log-concave distributions with the unadjusted langevin algorithm,'' \emph{arXiv preprint arXiv:1605.01559}, 2016.

\bibitem{durmus2017nonasymptotic}
------, ``Nonasymptotic convergence analysis for the unadjusted langevin algorithm,'' \emph{The Annals of Applied Probability}, vol.~27, no.~3, pp. 1551--1587, 2017.

\bibitem{durmus2019high}
------, ``High-dimensional bayesian inference via the unadjusted langevin algorithm,'' \emph{Bernoulli}, vol.~25, no.~4A, pp. 2854--2882, 2019.

\bibitem{welling2011bayesian}
M.~Welling and Y.~W. Teh, ``Bayesian learning via stochastic gradient langevin dynamics,'' in \emph{Proceedings of the 28th international conference on machine learning (ICML-11)}, 2011, pp. 681--688.

\bibitem{ram2009asynchronous}
S.~S. Ram, A.~Nedi{\'c}, and V.~V. Veeravalli, ``Asynchronous gossip algorithms for stochastic optimization,'' in \emph{Proceedings of the 48h IEEE Conference on Decision and Control (CDC) held jointly with 2009 28th Chinese Control Conference}.\hskip 1em plus 0.5em minus 0.4em\relax IEEE, 2009, pp. 3581--3586.

\bibitem{teh2016consistency}
Y.~Teh, A.~Thi{\'e}ry, and S.~Vollmer, ``Consistency and fluctuations for stochastic gradient langevin dynamics,'' \emph{Journal of Machine Learning Research}, vol.~17, 2016.

\bibitem{cormen2022introduction}
T.~H. Cormen, C.~E. Leiserson, R.~L. Rivest, and C.~Stein, \emph{Introduction to algorithms}.\hskip 1em plus 0.5em minus 0.4em\relax MIT press, 2022.

\end{thebibliography}
